\documentclass[preprint,12pt]{elsarticle}

\usepackage{amssymb}
\usepackage{CJK}
\usepackage[cmex10]{amsmath}
\usepackage{amsmath,amsfonts,amssymb,amsthm,dsfont,stmaryrd}
\usepackage{hyperref}
\usepackage{color}
\usepackage{tikz}
\usepackage{multirow}
\usepackage{rotating}

\newtheorem{observation}{Observation}[section]
\newtheorem{theorem}{Theorem}[section]

\newtheorem{lemma}{Lemma}[section]
\newtheorem{corollary}{Corollary}[section]

\newtheorem{definition}{Definition}[section]

\newtheorem{example}{Example}[section]

\def\M{\mathcal{M}}

\def\afCEC{afCEC }
\def\m{\mathrm{m}}

\def\R{\mathbb{R}}

\def\det{{\mathrm{det}}}

\def\mle{\mathrm{MLE}}
\def\cec{\mathrm{CEC}}
\def\cov{{\mathrm{cov}}}
\def\m{{\mathrm m}}

\newcommand{\argmin}{\operatornamewithlimits{argmin}}
\def\-{\hat}

\def\m{\mathrm{m}}

\def\A{\mathcal{A}}

\def\F{\mathcal{F}}
\def\G{\mathcal{G}}
\def\C{\mathcal{C}}

\def\v{\mathrm{v}}

\def\v{\mathrm{v}}

\def\ssigma{\Sigma}
\usepackage{lscape}

\def\y{\mathrm{y}}
\def\x{\mathrm{x}}

\newcommand{\ppp}[2]{ X_{#1}^{#2} }

\def\1{\mathds{1}}

\def\det{\mathrm{det}}

\def\cov{\mathrm{cov}}
\def\mean{\mathrm{mean}}

\numberwithin{equation}{section}

\usepackage{algorithmic}
\usepackage{algorithm}

\usepackage{epstopdf}
\usepackage[tight,footnotesize]{subfigure}

\journal{Pattern Recognition}

\begin{document}

\begin{frontmatter}



\title{Active Function Cross-Entropy Clustering}

\author{P. Spurek}
\ead{przemyslaw.spurek@ii.uj.edu.pl}
\author{J. Tabor}
\ead{jacek.tabor@ii.uj.edu.pl}
\author{P. Markowicz}
\ead{pawel.markowicz@ii.uj.edu.pl}

\address{
Faculty of Mathematics and Computer Science, 
Jagiellonian University, 
\L ojasiewicza 6, 
30-348 Krak\'ow, 
Poland}





\begin{abstract}
Gaussian Mixture Models (GMM) have found many applications in density estimation and data clustering. However,
the model does not adapt well to curved and strongly nonlinear data. Recently there appeared an improvement called AcaGMM (Active curve axis Gaussian Mixture Model),
which fits Gaussians along curves using an EM-like (Expectation Maximization) approach. 

Using the ideas standing behind AcaGMM, we 
build an alternative active function model of clustering, which
has some advantages over AcaGMM. In particular it is naturally defined in arbitrary dimensions and enables an easy adaptation to clustering of complicated datasets
along the predefined family of functions. Moreover, it does not need external methods to determine the 
number of clusters as it automatically reduces the number of groups on-line.

%
\end{abstract}


\end{frontmatter}
\section{Introduction}

Clustering plays a basic role in many parts
of data engineering, pattern recognition and image analysis
\cite{Clu,Clus,jain1999,jain2010,xu2009clustering}. One of the most important is Gaussian
Mixture Models \cite{mclachlan2007algorithm,mclachlan2004finite,Dubes,hinton1997modeling}.
It is hard to overestimate the role of GMM in computer science \cite{mclachlan2007algorithm,mclachlan2004finite,Dubes,hinton1997modeling}, including object
detection \cite{kumar2003man,campbell1997linear,dasgupta1998detecting,huang1998extensions,samuelsson2004waveform,figueiredo2002unsupervised}, object tracking \cite{mckenna1999tracking,xiong2002improved}, learning and modelling \cite{moghaddam1997probabilistic,samuelsson2004waveform}, feature selection \cite{law2004simultaneous,valente2004variational}, classification \cite{povinelli2004time,mukherjee1998three} or statistical background subtraction \cite{stauffer1999adaptive,hayman2003statistical,basu2002statistical}.

GMM accommodates data of varied structure, e.g. the component distributions can concentrate around surfaces of lower dimension obtained by principal components (PCA) \cite{Jol}. However, it often happens that clusters are concentrated around lower dimensional manifolds which are not linear. 
Since one non-Gaussian component can often be approximated by several Gaussian ones \cite{fraley1998many}, these clusters are in practice represented by introducing more Gaussian components which can be seen as a form of piecewise linear approximation, see Fig.~\ref{fig:acaCEC}.
Due to the intrinsic linearity of the Gaussian model, when there are nonlinear manifolds in the data cloud, it is natural that many components are required and the fitting error is large. Consequently, the constructed model does not reflect optimally the internal structure of the data.
A similar result gives Cross Entropy Clustering approach, compare Fig \ref{fig:letterb1} and \ref{fig:letterb2}.

\begin{figure}[htp]
\begin{center}
\subfigure[Level set for classical Gaussian density.]
{\label{fig:eGauss}
\includegraphics[width=1.5in]{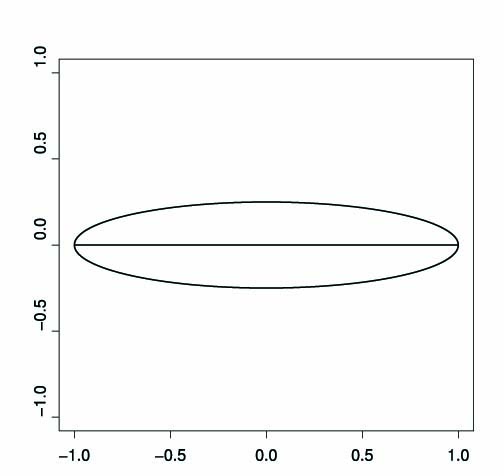} } \qquad
\subfigure[Level set of AcaGMM Gaussian model.]
{\label{fig:eGauss}
\includegraphics[width=1.5in]{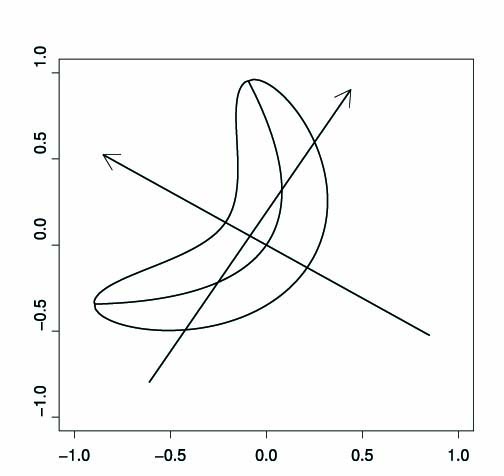} } 
\end{center}
\caption{Comparison of level-sets generated by classical Gaussian density and AcaGMM model.}
\label{fig:acaCEC}
\end{figure}

There are several methods attempting to solve the problem of fitting nonlinear manifolds, e.g. principal curves and principal surfaces \cite{kegl1999principal,hastie1989principal,leblanc1994adaptive}. Principal curves/surfaces algorithms are typically capable of expressing a single complex manifold.
In \cite{zhang2005active} the authors present an adaptation of the Gaussian Mixture Model called Active curve axis Gaussian Mixture Models (AcaGMM),
which uses a nonlinear curved Gaussian probability model in clustering. In its basic version it works with data on the plane and
adapts to the quadratic curves. In other words AcaGMM uses a wider class then typical Gaussians -- namely Gaussians which are curved over parabolas.

Since our paper aims at solving the same task as AcaGMM, let us 
first explain the method and present
the typical steps behind it. 
First, using an additional tool, the authors find 
the ``right'' number of clusters (one of the possible methods is given in \cite{zhang2004competitive}, however, one can also use \cite{tabor2014cross}). Then for each cluster the PCA algorithm is applied to determine the reasonable basis, and a Gaussian curved along the optimal parabola is used. The coordinate system is nonlinear, see Fig. \ref{fig:letterb3} (the $y$ coordinate is chosen
as a distance from the parabola, and $x$ is
the length on the parabola from the projected point to 
the parabola's vertex).
AcaGMM has found applications in particularly in human hand motion recognition \cite{ju2011unified}. It can also be fuzzified \cite{ju2012fuzzy}.

\begin{figure}[htp]
\begin{center}
\subfigure[]
{\label{fig:letterb1}
\includegraphics[width=1.1in]{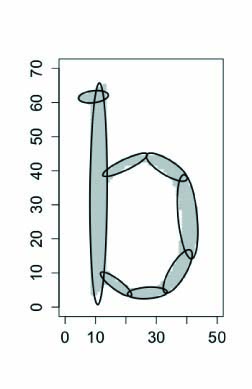}}  \quad
\subfigure[]
{\label{fig:letterb2}
\includegraphics[width=1.1in]{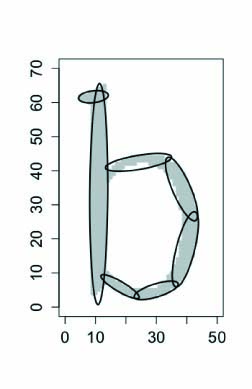}}  \quad
\subfigure[]
{\label{fig:letterb3}
\includegraphics[width=1.1in]{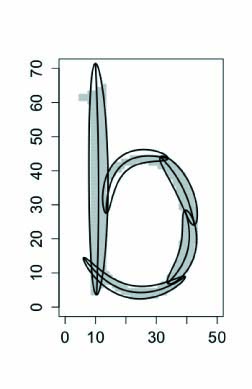}}  \quad
\subfigure[]
{\label{fig:letterb4}
\includegraphics[width=1.1in]{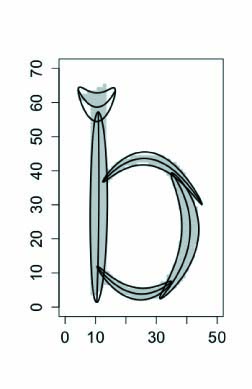}}  \quad
\end{center}
\caption{ Fitting a b-type set by using (a) GMM, (b) CEC, (c) AcaGMM, (d) afCEC.}
\label{fig:letterb}
\end{figure}

AcaGMM works well in practice, however, it has some limitations. The model is naturally restricted to quadratic functions as the nonlinear coordinate system requires the 
projection onto the graph and length of the curve. The use of the method in higher dimensional case, although possible, is practically rather limited. 
%
%
Moreover, AcaGMM is not a theoretically based density model (see Appendix for the detailed explanation), and therefore it is not in fact formally EM based, but only uses its optimization algorithm. Consequently, contrary to the classical EM \cite{EM2, EM3}, the MLE cost function does not necessarily decrease with iterations. Let us recall that in general EM aims at finding $p_1,\ldots,p_k \geq 0$, $\sum_{i=1}^k p_i=1$ and $f_1,\ldots, f_k$ Gaussian densities (where $k$ is given beforehand and denotes the number of densities which convex combination builds the desired density model) such that the convex combination
$$
f:=p_1 f_1 +\ldots p_k f_k
$$
optimally approximates the scatter of our data $X=\{x_1,\ldots,x_n\}$ with respect to MLE cost function
\begin{equation} \label{eq1}
\mle(f,X):=-\sum_{l=1}^n \ln(p_1 f_1(x_l) +\ldots +p_n f_n(x_l)).
\end{equation}
The EM procedure consists of the Expectation and Maximization steps. While the Expectation step is relatively simple, the Maximization
usually needs complicated numerical optimization even 
for relatively simple Gaussian models \cite{Ma,celeux1995gaussian,banfield1993model}.

\begin{figure}[htp]
\begin{center}
\includegraphics[width=2.3in]{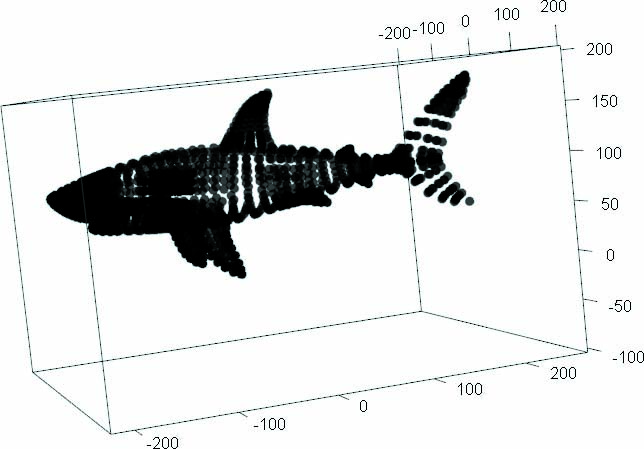} 
\includegraphics[width=2.3in]{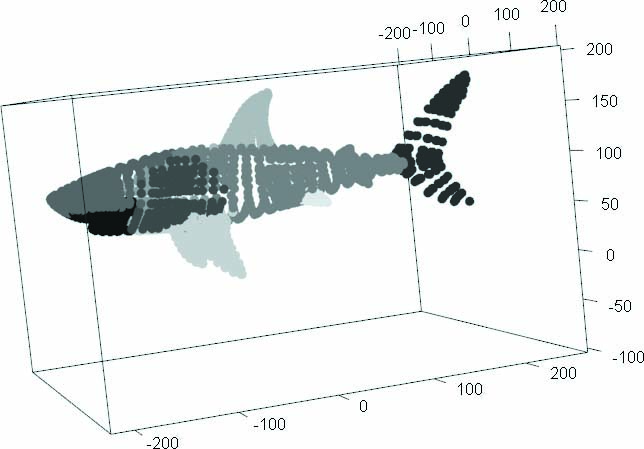} 
\includegraphics[width=2.3in]{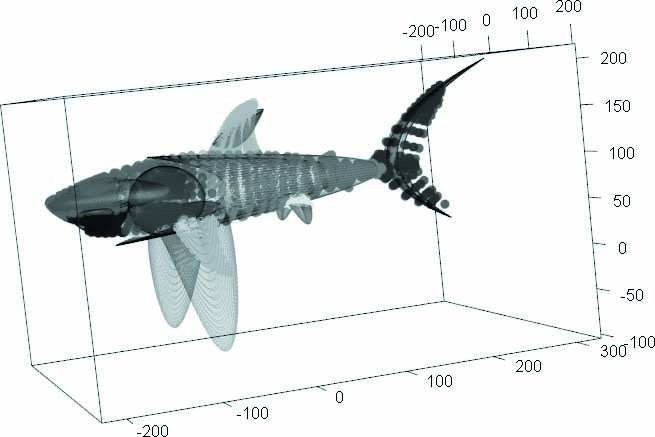} 
\includegraphics[width=2.3in]{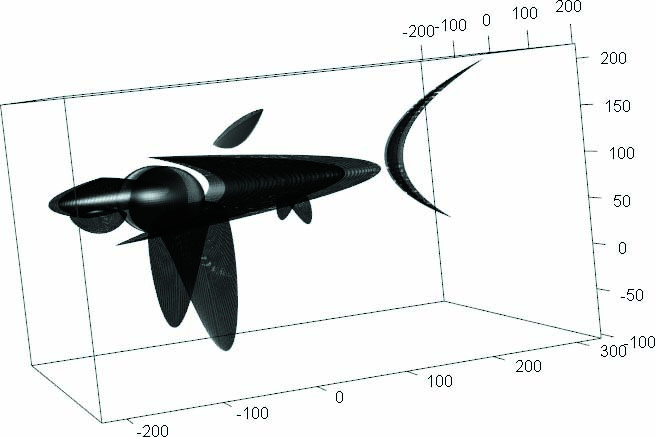}
\end{center}
\caption{ Result of \afCEC algorithm in the case of a 3D shark-type set.}
\label{fig:3d_1}
\end{figure}

In this paper we propose the \afCEC method which is based on the CEC model, instead of the Expectation Maximization (EM) and Gaussian density model in a curvilinear coordinate system. A goal of CEC is to minimize the
cost function, which is a minor modification of that given in \eqref{eq1}
by substituting sum with maximum:
\begin{equation} \label{eq2}
\cec(f,X):=-\sum_{l=1}^n \ln(\max(p_1 f_1(x_l),\ldots,p_n f_n(x_l))).
\end{equation}
Instead of focusing on the density estimation as its
main task, CEC aims itself directly to the clustering problem. It occurs that at the small cost of minimally worse density approximation \cite{tabor2014cross} we gain speed in implementation\footnote{We can often use the Hartigan approach to clustering which is faster and typically finds better minima.} and the ease of using more complicated density models. Roughly speaking, the advantage is obtained because models do not mix with each other, since
we take the maximum instead of sum.

Consequently, we are able to construct an algorithm which is easy to adapt to the higher dimensional case.
The results of \afCEC and AcaGMM are similar on the plane, compare Fig. \ref{fig:letterb3} and Fig. \ref{fig:letterb4}. 
The effect of our algorithm in $\R^3$ on a shark-type set \cite{bronstein2008numerical,bronstein2006efficient} is shown in Fig. \ref{fig:3d_1}.

The \afCEC method is able to reduce unnecessary clusters.
In Fig. \ref{fig:dog} we present a convergence process of \afCEC with initial number of clusters $k=10$, which is reduced to $k=5$.

\begin{figure}[htp]
\begin{center}
\includegraphics[width=0.85in]{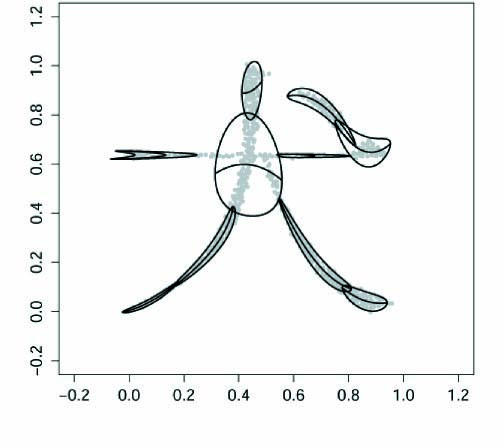} 
\includegraphics[width=0.85in]{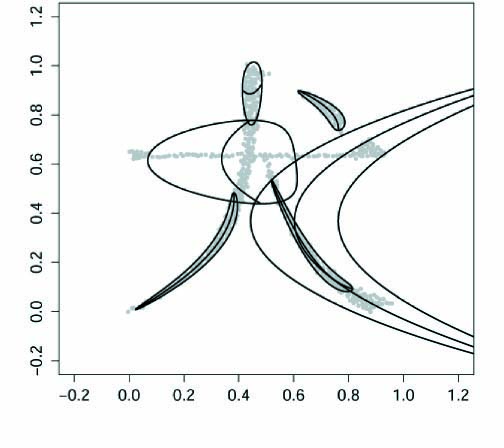} 
\includegraphics[width=0.85in]{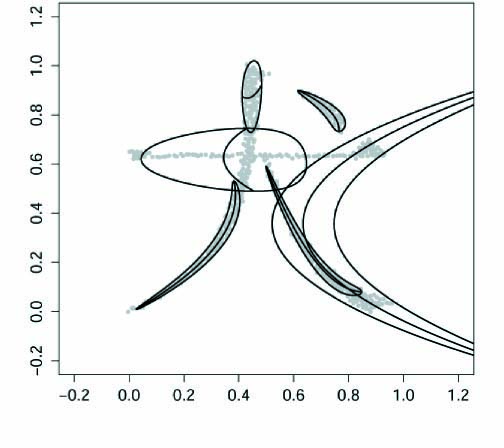} 
\includegraphics[width=0.85in]{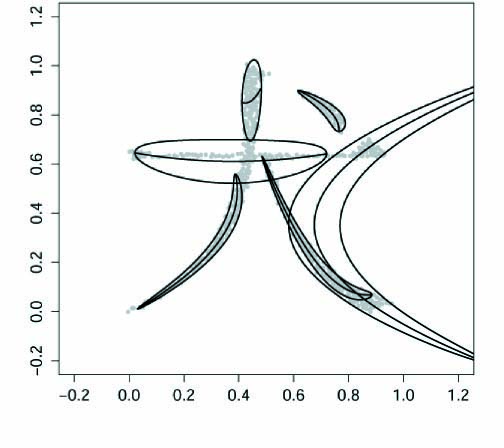}
\includegraphics[width=0.85in]{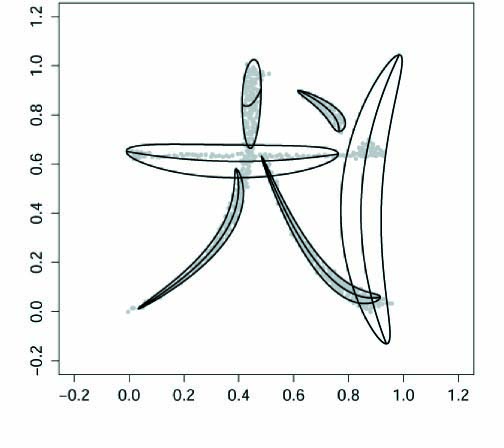}
\includegraphics[width=0.85in]{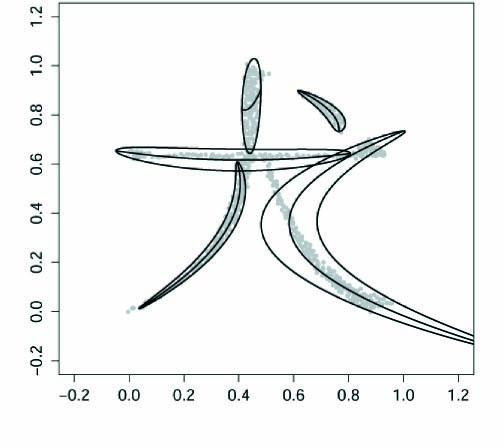}
\includegraphics[width=0.85in]{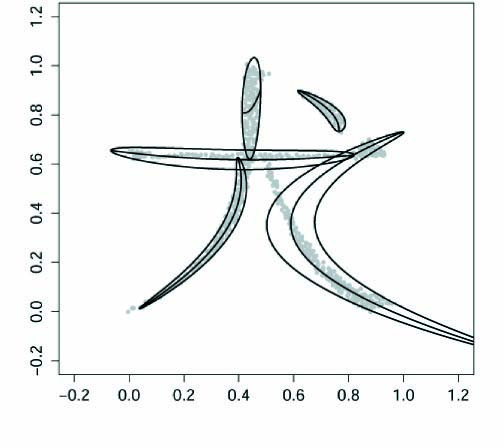} 
\includegraphics[width=0.85in]{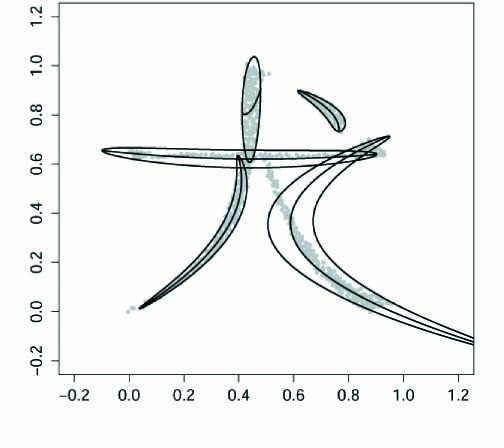} 
\includegraphics[width=0.85in]{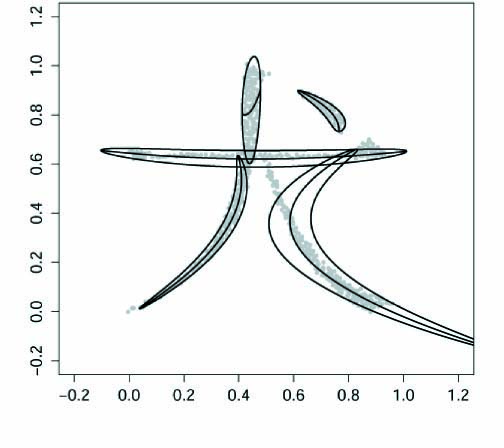} 
\includegraphics[width=0.85in]{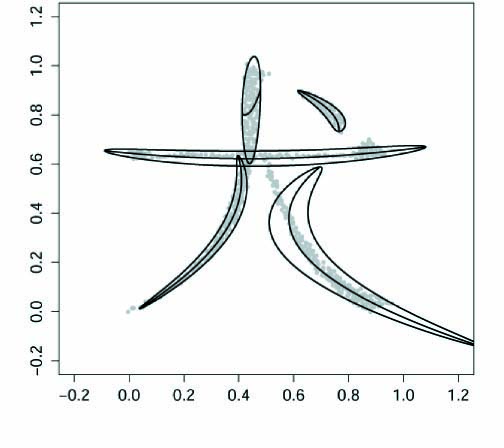} 
\includegraphics[width=0.85in]{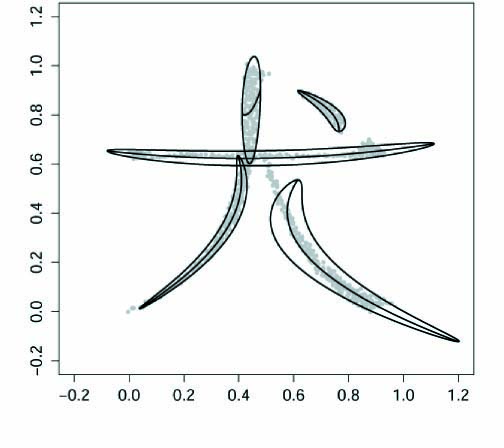}
\includegraphics[width=0.85in]{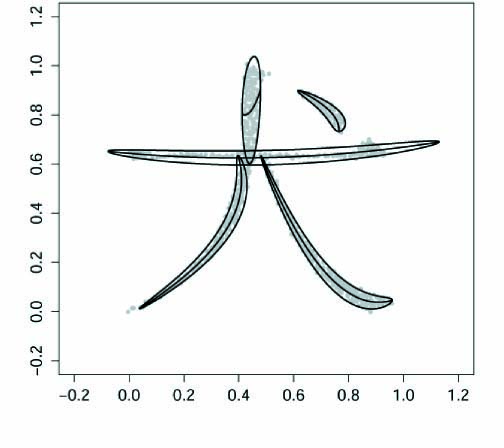}
\end{center}
\caption{ A convergence process of \afCEC on a Chinese character with initial $k=10$, which is reduced to $k=5$.}
\label{fig:dog}
\end{figure}

This paper is arranged as follows. In the next section the theoretical background of the density model will be presented. Since AcaGMM works in $\R^2$ only for parabolas we start with a similar situation. Then we describe a general model for data in $\R^d$. In the third chapter we present the theoretical background of the \afCEC method. In particular, we prove that the cost function decreases in every iteration, see Theorem \ref{the:ener}. The last chapter presents numerical experiments. In appendix we include details of the description of the AcaGMM model. 

\section{$f$-adapted Gaussian density}

In this section, the $f$-adapted Gaussian distribution, where $f \in \C(\R^{d-1},\R)$ is a continuous function, will be presented. The goal of this approach is to transform a normal distribution (which assumes the intrinsic linearity of the model) to the case of curves (or more generally to manifolds), which are given by the graph of the function~$f$. The above model will be used in the \afCEC method.

\subsection{Toy example in $\R^2$}

Since AcaGMM works in the two-dimensional case (in higher dimensional ones the authors use PCA to reduce problems to 2D) with parabolas ($f(x) = ax^2 + b$ for $a,b \in \R$), we start from comparison AcaGMM and our model in such a case.
Let $f(x) = ax^2 + b$ for $a,b \in \R$ be given. The two dimensional Gaussian density for $\m^T=[m_1,m_2]$ and covariance matrix $
\Sigma=
\begin{bmatrix}
\sigma_{1} & 0 \\
0 & \sigma_{2}
\end{bmatrix}
$ is given by the following formula
\begin{equation}\label{gauss_2d}
N(\m,\Sigma)( \x ) =N(m_1,\sigma_1^2)(x_1) \cdot N(m_2,\sigma_2^2)(x_2),
\end{equation}
where in the one dimensional case we have
$$
N(m,\sigma^2)(x)=
\frac{1}{\sqrt{2 \pi} \sigma}
\exp\left(-\frac{|x-m |^2}{2 \sigma^2} \right) \mbox{ for } m,\sigma \in \R.
$$

Let $\x= [x_1,x_2]^T \in \R^2$ be given. The AcaGMM approach uses the orthogonal projection of the point $\x$ onto the parabola $f$ which is denoted by $p_f(\x)$ and the arc length between $p_f(\x)$ and $\m$ which is denoted by $l_f(p_f(\x),\m)$. Consequently the AcaGMM function is given by
\begin{equation}\label{acagmm_fun}
N(\m,\Sigma,f)( \x) = \tfrac{1}{\sqrt{2 \pi } \sigma_1} \exp{\left(-\tfrac{l(\x,\m_1)^2}{2\sigma_1^2} \right)} \cdot 
\tfrac{1}{\sqrt{2 \pi } \sigma_2} \exp{\left(-\tfrac{ \| p_f(\x) - \x \|^2}{2\sigma_2^2} \right)}.
\end{equation}

This approach is very intuitive but it causes two basic problems. It is very hard (or even impossible) to give explicit formulas for orthogonal projection and arc length for more complicated curves in higher dimensional spaces. Calculations are complicated (from the numerical point of view), consequently the field of possible generalizations of AcaGMM is limited. Moreover, the function which was used in AcaGMM, see formula (\ref{acagmm_fun}), is not a density. The Jacobian of the respective transformation was not included (see Appendix).

In our paper we use a simpler approach, which is based on the Euclidean norm and the following formula for the density function $f$: 
\begin{equation}\label{gauss_curve_2d}
N(\m,\Sigma,f)( [x_1,x_2]) =N(m_1,\sigma_1^2)(x_1)\cdot N(m_2,\sigma_2^2)(x_2-f(x_1)).
\end{equation}
 Since we do not use orthogonal projection and arc length, it is easy to calculate the parameters of our generalized Gaussian distribution, see Fig.~\ref{fig:el}.
 
\begin{figure}[htp]
\begin{center}
\subfigure[]
{\label{fig:eGauss}
\includegraphics[width=1.5in]{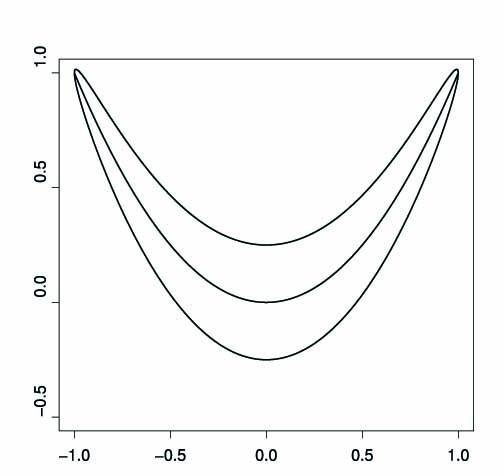} } \quad
\subfigure[]
{\label{fig:eGauss}
\includegraphics[width=1.5in]{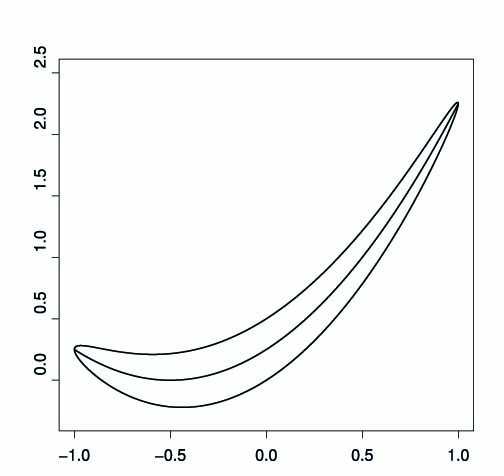} } 
\end{center}
\caption{Density level-sets generated by the $f$-adapted Gaussian model.}
\label{fig:el}
\end{figure}

The practical difference in $\R^2$ between AcaGMM and our approach is quite small\footnote{In our case we use the parabola $ax^2+bx+c$ instead $ax^2+c$ since our method does not apply the change of coordinates given by PCA.}, see Fig. \ref{fig:one_claster}. Nevertheless, our model is more flexible, as
we can use an arbitrary class of functions for which least squares methods work.

\begin{figure}[htp]
\begin{center}
\subfigure[The AcaGMM method.]
{\label{fig:fu1}
\includegraphics[width=1.3in]{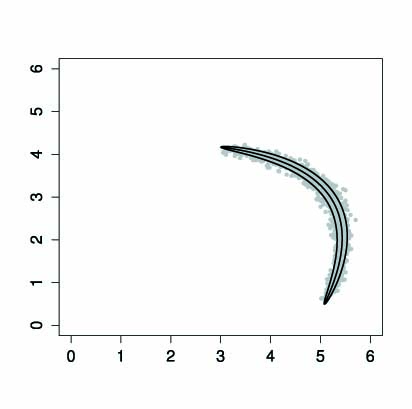}
\includegraphics[width=1.3in]{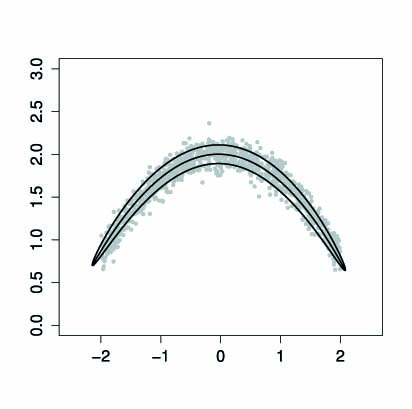}
\includegraphics[width=1.3in]{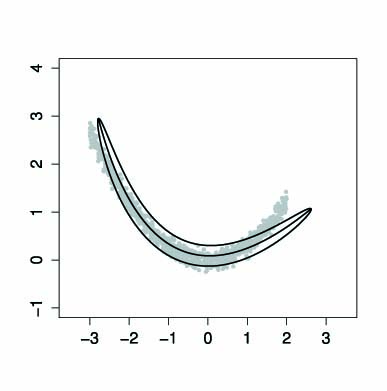}} \\ 
\subfigure[The afCEC method.]
{\label{fig:fu2}
\includegraphics[width=1.3in]{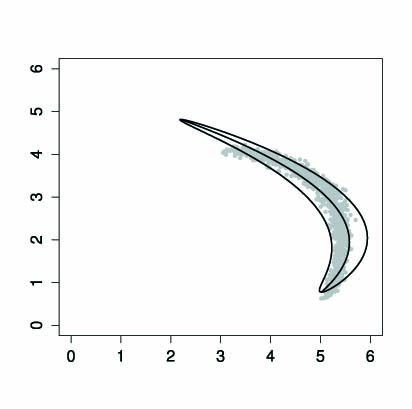}
\includegraphics[width=1.3in]{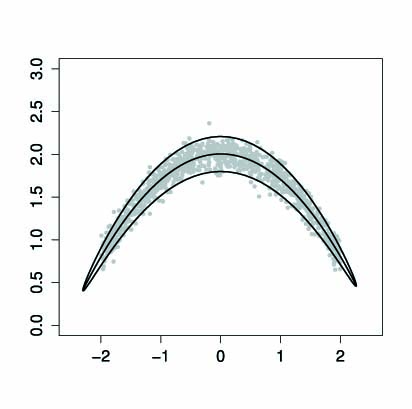}
\includegraphics[width=1.3in]{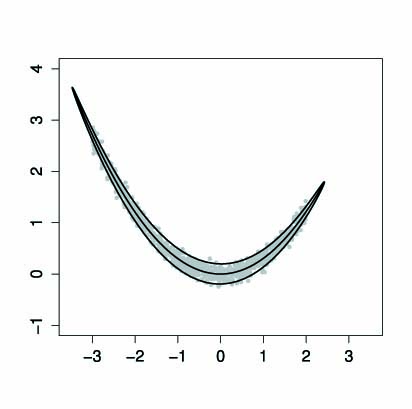}} 
\end{center}
\caption{Ellipses generated by AcaGMM and afCEC.}
\label{fig:one_claster}
\end{figure}

\subsection{$f$-adapted Gaussian density}

In this subsection, the general notion of $f$-adapted Gaussian will be presented. Let us recall that the standard Gaussian density in $\R^d$ is defined by 
$$
N(\m,\Sigma)(\x):=\frac{1}{(2\pi)^{d/2} \det(\Sigma)^{1/2}} \exp \left(-\frac{1}{2}\|\x-\m\|^2_{\Sigma}\right),
$$
where $\m$ denotes the mean, $\Sigma$ is the covariance matrix and
$
\|v\|^2_{\Sigma}:=v^T \Sigma^{-1} v
$
is the square of the Mahalanobis norm.

\begin{figure}[htp]
\begin{center}
\subfigure[$f(x)=0$]
{\label{fig:eGauss}
\includegraphics[width=1.2in]{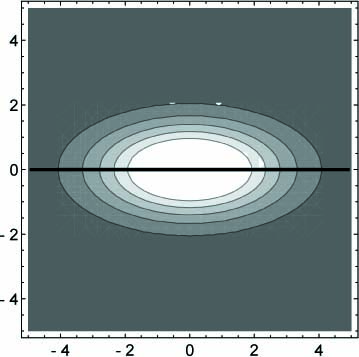}} 
\subfigure[$f(x)=x$]
{\label{fig:eGaca}
\includegraphics[width=1.2in]{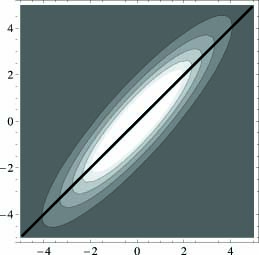}}
\subfigure[$f(x)=\frac{1}{8}x^2$]
{\label{fig:eGcec}
\includegraphics[width=1.2in]{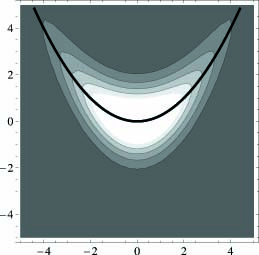}} 
\subfigure[$f(x)=\frac{1}{16}x^3$]{\label{fig:eGcec}
\includegraphics[width=1.2in]{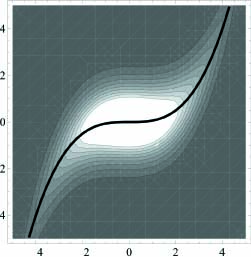}} 
\end{center}
\caption{Level sets for $f$-adapted Gaussian distribution.}
\label{fig:e}
\end{figure}

In our work we use a multidimensional Gaussian density in a curvilinear coordinate system which is spread along the function $f \colon \R^{d-1} \to \R $ ($f$-adapted Gaussian density). We treat one of the variables (for simplicity, the last one) separately. In such a case we consider only those
$\Sigma \in \M_{d}(\R)$ (where $\M_{d}(\R)$ denotes the set of $d$-dimensional square matrices) which have the diagonal block matrix form
$$
\Sigma=
\begin{bmatrix}
\Sigma_{\hat d} & 0 \\
0 & \ssigma_{d}
\end{bmatrix},
$$
where $\Sigma_{\hat d} \in \M_{d-1}(\R)$ and $\ssigma_{d} \in \R$.
For $\x = (x_1,\ldots,x_n) \in \R^d$ and $k=1,\ldots,n$ we will use the notation 
$$
\x_{\hat k} := (x_1,\ldots,x_{k-1},x_{k+1},\ldots,x_{d}) \in \R^{d-1}.
$$ 
For $X \subset \R^{d}$, we denote 
$
X_{\hat k} := \{ \x_{\hat k} \colon \x \in X \},
$ 
the set containing vectors from $X$ with removed $k$ coordinate, and
$
X_{k} := \{ x_{k} \colon \x \in X \}.
$
For a function $f \colon \R^{d-1} \to \R$, we denote
$$
\ppp{k}{f} := \{ f(\x_{\hat k}) - x_k \colon \x \in X \}.
$$

\begin{definition}
Let $f \in \C(\R^{d-1},\R) $, $\Sigma_{\hat d} \in \M_{d-1}(\R)$, $\ssigma_{d} \in \R$, $\m \in \R^{d}$ be given.
The $f$-adapted Gaussian density for 
$\Sigma_{\hat d}$, $\ssigma_{d}$
and $\m$ is defined as follows
\begin{equation} \label{gauss_2ddo}
\begin{array}{l}
N(\m,\Sigma_{\hat d},\ssigma_d,f)(\x)= N(\m_{\hat d},\Sigma_{\hat d})(\x_{\hat d}) \cdot N(m_{d},\ssigma_{d})(x_{d} - f(\x_{\hat d})) 
\end{array}
\end{equation}
\end{definition}

Level sets for $f$-adapted Gaussian distributions with different types of functions are presented in Fig.~\ref{fig:e}.

\begin{observation}
The $f$-adapted Gaussian function $N(\m,\Sigma_{\hat d},\ssigma_d,f)(\x)$, 
where 
$f \in \C(\R^{d-1},\R$), $\Sigma_{\hat d} \in \M_{d-1}(\R)$, $\ssigma_{d} \in \R$, $\m \in \R^{d}$
is a density.
\end{observation}
\begin{proof}
Let $N(\m,\Sigma)$ be a $d$-dimensional Gaussian density such, that 
$
\Sigma = 
\begin{bmatrix}
\Sigma_{\hat d} & 0 \\
0 & \Sigma_{d}
\end{bmatrix},
$
where $\Sigma_{\hat d} \in \M_{d-1}(\R)$, $\ssigma_{d} \in \R$, $\m \in \R^{d}$.

Let us consider a substitution 
$$
(y_1, \ldots , y_d) = (x_1, \ldots , x_{d-1} , x_d - f(\x_{\hat d})).
$$
In such a case, the Jacobian is equal to
$$
J(x_1, \ldots , x_{d})=
\det
\begin{bmatrix}
1 & 0 & \cdots & 0 & 0\\
0 & 1 & \cdots & 0 & 0\\
\vdots & \vdots & \ddots & \vdots & \vdots\\
0 & 0 & \cdots & 1 & 0\\
\frac{\partial f(\x_{\hat d})}{\partial x_1} & \frac{\partial f(\x_{\hat d})}{\partial x_2} & \cdots & \frac{\partial f(\x_{\hat d})}{\partial x_{d-1}} & 1
\end{bmatrix}
=
1.
$$
Consequently, $N(\m,\Sigma_{\hat d},\ssigma_d,f)(\x)$
is a density.
\end{proof}

We will use the family of all $d$-dimensional Gaussian densities $\G(\R^d)$.
Moreover, for $f \colon \R^{d-1} \to \R $, we will consider family of $f$-adapted Gaussian functions 
$$
\A_{f}(\R^{d-1},\R) := \left\{ N(\m,\Sigma_{\hat d},\ssigma_d,f) \colon \Sigma_{\hat d} \in \M_{d-1}(\R), \m \in \R^d \mbox{ and } \ssigma_{d} \in \R \right\}.
$$
For the family $\F\subset \C(\R^{d-1},\R)$, we define
$$
\A_{\F}(\R^{d-1},\R) = \bigcup_{f \in \F} \{ \A_{f}(\R^{d-1},\R) \}.
$$
We show that if $\F$ contains all linear transformations, then $\G(\R^d) \subset \A_{\F}(\R^{d-1},\R)$. Let us start with simple Lemma.

\begin{lemma}\label{lem:fam}
Let $\m \in \R^d$, $\Sigma_{\hat d} \in \M_{d-1}(\R)$, $\Sigma_{d}\in \R$ and $\v \in \R^{d-1}$ be given. Then for $
A=
\begin{bmatrix}
I_{d-1} & 0 \\
\v^{T} & -1 \\
\end{bmatrix}
$
we have
$$
N\left(A\m,A\begin{bmatrix}
\Sigma_{\hat d} & 0 \\
0 & \Sigma_{d}
\end{bmatrix} A^T\right)(\x) = N(\m,\Sigma_{\hat d},\Sigma_d,f)(\x),
$$
where
$f\colon \R^{d-1} \to \R$ such, that $f(\x) = \v^{T} \cdot \x$.
\end{lemma}
\begin{proof}
Let us denote $\Sigma = 
\begin{bmatrix}
\Sigma_{\hat d} & 0 \\
0 & \Sigma_{d}
\end{bmatrix}
$ and $\m^T=[\m_{\hat d},m_{d}]$, then we have
$$
N(A\m,A\Sigma A^T)(\x) =N\left(\begin{bmatrix}
I_{d-1} & 0 \\
\v^{T} & -1 \\
\end{bmatrix}
\begin{bmatrix}
\m_{\hat d} \\
m_d \\
\end{bmatrix}
,\begin{bmatrix}
I_{d-1} & 0 \\
\v^{T} & -1 \\
\end{bmatrix} \begin{bmatrix}
\Sigma_{\hat d} & 0 \\
0 & \Sigma_{d}
\end{bmatrix} \begin{bmatrix}
I_{d-1} & 0 \\
\v^{T} & -1 \\
\end{bmatrix}^T\right)(\x)=
$$
$$
=N\left(
\begin{bmatrix}
\m_{\hat d} \\
\v^T \m_{\hat d} - m_d \\
\end{bmatrix}
, \begin{bmatrix}
\Sigma_{\hat d} & 0 \\
\v^T \Sigma_{\hat d} & -\Sigma_{d}
\end{bmatrix} \begin{bmatrix}
I_{d-1} & \v \\
0 & -1 \\
\end{bmatrix}\right)(\x)=
$$
$$
= N\left(
\begin{bmatrix}
\m_{\hat d} \\
\v^T \m_{\hat d} - m_d \\
\end{bmatrix}
, \begin{bmatrix}
\Sigma_{\hat d} & \Sigma_{\hat d}\v \\
\v^T \Sigma_{\hat d} & \v^T \Sigma_{\hat d}\v + \Sigma_{d}
\end{bmatrix} \right)(\x).
$$
It is easy to show that
$$
(A\Sigma A^T)^{-1}=
\begin{bmatrix}
\Sigma_{\hat d} & \Sigma_{\hat d}\v \\
\v^T \Sigma_{\hat d} & \v^T \Sigma_{\hat d}\v + \Sigma_{d}
\end{bmatrix}^{-1}=$$
$$
=\begin{bmatrix}
\Sigma_{\hat d}^{-1} & 0 \\
0 & 0
\end{bmatrix}+
\Sigma_{d}^{-1}\begin{bmatrix}
\v\v^T & -\v \\
-\v^T & 1
\end{bmatrix}=
\begin{bmatrix}
\Sigma_{\hat d}^{-1} & 0 \\
0 & 0
\end{bmatrix}+
\Sigma_{d}^{-1}
\begin{bmatrix}
-\v \\
1
\end{bmatrix}
[-\v^T , 1].
$$
Therefore we have
$$
[\x_{\hat d}^T, x_d]
(A\Sigma A^T)^{-1}
\begin{bmatrix}
\x_{\hat d} \\ x_{d}
\end{bmatrix}=%
[\x_{\hat d}^T, x_d]
\left( \begin{bmatrix}
\Sigma_{\hat d}^{-1} & 0 \\
0 & 0
\end{bmatrix}+
\Sigma_{d}^{-1}
\begin{bmatrix}
-\v \\
1
\end{bmatrix}
[-\v^T , 1] \right)^{-1}
\begin{bmatrix}
\x_{\hat d} \\ x_{d}
\end{bmatrix}=
$$
$$
=\x_{\hat d}^T \Sigma_{\hat d}^{-1} \x_{\hat d}+(x_{d}-\x_{\hat d}^T \v) \Sigma_{ d}^{-1} (x_{d}-\x_{\hat d} \v^T)=%
[\x_{\hat d}^T, x_{d}-\v^T \x_{\hat d}]
\begin{bmatrix}
\Sigma_{\hat d} & 0 \\
0 & \Sigma_{d}
\end{bmatrix}^{-1}
\begin{bmatrix}
\x_{\hat d} \\ x_{d}-\v^T \x_{\hat d} 
\end{bmatrix}.
$$
As a simple consequence we obtain the assertion of the Lemma.
\end{proof}

Now we show that $f$-adapted Gaussian densities are an extension of the classical Gaussian model. 

\begin{theorem}\label{the:gr}
Let $\F= \{ f \colon \R^{d-1} \to \R \colon f(\x)=\v^T \cdot \x \mbox{ for } \v \in \R^{d-1} \}$ be the family of all linear transformations from $\R^{d-1}$ into $\R$. Then 
$$
\A_{\F}(\R^{d-1},\R) = \G(\R^d).
$$
\end{theorem}
\begin{proof}
To prove the assertion, we first show the following inclusion:
$$
\A_{\F}(\R^{d-1},\R) \subset \G(\R^d).
$$
Let $\m \in \R^d$, $\Sigma = 
\begin{bmatrix}
\Sigma_{\hat d} & 0 \\
0 & \Sigma_{d}
\end{bmatrix}$ (where $\Sigma_{\hat d}\in \M_{d-1}$, $\ssigma_{d} \in \R$), $\v \in \R^{d-1}$ and $f(\x)=\v^T \cdot \x$ be given and let $N(\m,\Sigma_{\hat d},\Sigma_{d},f) \in \A_{\F}(\R^{d-1},\R)$.
Thanks to Lemma~\ref{lem:fam} for 
$
A=
\begin{bmatrix}
I & 0 \\
\v^{T} & -1 \\
\end{bmatrix}
$, we have
$$
N(\m,\Sigma_{\hat d},\Sigma_{d},f) = N(A\m,A\Sigma A^T) \in \G(\R).
$$
We now show the opposite inclusion 
$$
\G(\R^d) \subset \A_{\F}(\R^{d-1},\R). 
$$
Let
$
\Sigma = 
\begin{bmatrix}
\Sigma_{11} & \v \\
\v^T & \ssigma_{22}
\end{bmatrix} \in \M_{d}(\R) 
$
and $\m \in \R^d$
be given and let
$ N(\m,\Sigma) \in \G(\R^d)$. 
We put $\Sigma_{\hat d} = \Sigma_{11}$, $\Sigma_{d} = -\v^{T} \Sigma_{11}^{-1} \v + \ssigma_{22}$, $f(\x)=\Sigma_{11}^{-1} \v^T \x$ and
$
A=
\begin{bmatrix}
I & 0 \\
\v^T \Sigma_{11}^{-1} & -1 \\
\end{bmatrix}
$. 
Thanks to Lemma \ref{lem:fam}, we have 
$$
N\left[A^{-1}\m,\Sigma_{\hat d}, \ssigma_{d}, f\right] = 
N\left( \m,
\begin{bmatrix}
I & 0 \\
\v^{T}\Sigma_{11}^{-1} & -1
\end{bmatrix}
\begin{bmatrix}
\Sigma_{11} & 0 \\
0 & -\v^{T} \Sigma_{11}^{-1} \v + \ssigma_{22}
\end{bmatrix}
\begin{bmatrix}
I & \Sigma_{11}^{-1}\v \\
0 & -1
\end{bmatrix} \right)= 
$$ 
$$
=N\left(\m,
\begin{bmatrix}
\Sigma_{11} & 0 \\
\v^{T} & \v^{T} \Sigma_{11}^{-1} \v - \ssigma_{22}
\end{bmatrix}
\begin{bmatrix}
I & \Sigma_{11}^{-1}\v \\
0 & -1
\end{bmatrix} \right)=
N\left(\m,
\begin{bmatrix}
\Sigma_{11} & \v \\
\v^{T} & \ssigma_{22}
\end{bmatrix} \right).
$$
Consequently
$$
N\left(\m,\Sigma \right)= N\left(A^{-1}\m,\Sigma_{\hat d}, \ssigma_{d}, f\right)\in \A_{\F}(\R),
$$ 
what finished the proof.
\end{proof}

The following observation is a corollary of Theorem \ref{the:gr}.

\begin{corollary}
Let $\F\subset \C(\R^{d-1},\R)$ contains the family of all linear transformations from $\R^{d-1}$ into $\R$. Then 
$$
\G(\R^d) \subset \A_{\F}(\R^{d-1},\R).
$$
\end{corollary}

Consequently \afCEC is a natural extension of the classical CEC algorithm. If we consider $\F$ containing only linear transformations, we obtain exactly the CEC algorithm. On the other hand, for wider classes of functions we detect more general clusters, which describe groups concentrated around manifolds which are not necessarily linear. 

\section{Theoretical background of \afCEC}

In this section the theoretical background of \afCEC will be presented.
First, we introduce the cost function which will be minimized by the algorithm. Then we prove that the optimal function which describes each cluster can be obtained by least square regression \cite{bjorck1996numerical}. 
We will end by describing the full algorithm of afCEC. 

Our method is based on the CEC approach. Therefore, we start with a short introduction to the method (for a more detailed explanation
we refer the reader to \cite{tabor2014cross}). To explain CEC we need to introduce the cost function which we want to minimize. 
In the case of splitting of $X \subset \R^d$
into $X_1, \ldots , X_k$ such that
elements of $X_i$ we ”code” by function from family of all Gaussian densities $\G(\R^d)$, the mean code-length of a randomly chosen element $x$ equals
\begin{equation}\label{en:cec}
E(X_1, \ldots, X_k ; \G(\R^d) ):= \sum_{i=1}^{k} p_i \cdot \left( -\ln(p_i) + H^{\times}(X_i\|\G(\R^d) ) \right)
\end{equation}
where $p_i = \frac{|X_i|}{|X|}$.
The formula uses cross-entropy of a data set with respect to the family $\G(\R^d)$.

The aim of CEC is to find splitting of $\R^d$
into sets $X_i$
which minimize the function given in (\ref{en:cec}). Our goal is to calculate an explicit formula for the cost function in the case of $f$-adapted Gaussian densities.

\subsection{Cost function of one cluster}

In this section we will focus on the situation of one cluster $X$. 
In such a case we usually understand the data as a realization of a random variable. Consequently, as an estimator for the mean and covariance, we use 
$$
\mean(X) := \sum \limits_{\x \in X} \frac{\x}{n}, 
$$ 
$$
\cov(X) := \frac{1}{n} \sum \limits_{\x \in X} (\x-\mean(X))(\x-\mean(X))^T.
$$ 

As it was said, CEC uses cross-entropy of data set $X$ with respect to the Gaussian family $\G(\R^d)$.

\begin{theorem}\label{the:gaus_scross}
Let $X \subset \R^{d}$ be given. Then
$$
H^{\times}(X \| \G(\R^d)) = \inf_{g \in \G(\R^d)} H^{\times}(X\|g) = \frac{d}{2}\ln(2 \pi e) + \frac{1}{2} \ln(\det(\Sigma)),
$$
where 
$
\Sigma = \cov(X).
$
\end{theorem}

The CEC algorithm will be used for a family of $f$-adapted Gaussian densities.
In such a case the cost function is described by the following theorem.

\begin{theorem}\label{the:ener}
Let $X \subset \R^d$ and a function $f \in \C(\R^{d-1},\R)$ be given. Then
$$
H^{\times}(X \| \A_{f}(\R^{d-1},\R)) = \frac{d}{2} \ln( 2 \pi e ) + \frac{1}{2}\ln( \det(\Sigma_{\hat d}) ) + \frac{1}{2} \ln \left(\frac{1}{n} \sum_{\x \in X} (x_{d}-f(\x_{\hat d})-m_{d})^2 \right),
$$
where 
$
\Sigma_{\hat d} = \cov(X_{\hat k})
$
and
$
m_{d} = \mean(X_{k}).
$

\end{theorem}
\begin{proof}
Let $N(\m,\Sigma_{\hat d},\ssigma_{d},f)(x) \in \A_{f}(\R^{d-1},\R)$, where $\Sigma_{\hat d} \in \M_{d-1}(\R)$, $\ssigma_{d} \in \R$, $\m \in \R^{d}$ and $
\Sigma=
\begin{bmatrix}
\Sigma_{\hat d} & 0 \\
0 & \ssigma_{d}
\end{bmatrix}
$. The assertion of the proposition is a simple consequence of
$$
\begin{array}{l}
H^{\times}(X\|N(\m,\Sigma_{\hat d},\ssigma_{d},f)) =
-\frac{1}{|X|} \sum_{\x \in X} \ln(N(\m,\Sigma_{\hat d},\ssigma_{d},f)(\x))=
\\[6pt]
=-\frac{1}{|X|} \sum \limits_{\x \in X} \ln \left( N(\m_{\hat d},\Sigma_{\hat d})(\x_{\hat d}) \cdot N(m_{d},\ssigma_{d})(x_{d} - f(\x_{\hat d})) \right)=\\[6pt]
=-\frac{1}{|X|} \sum \limits_{\x \in X} \left( \ln \left( N(\m_{\hat d},\Sigma_{\hat d})(\x_{\hat d}) \right) + \ln\left( N(m_{d},\ssigma_{d})(x_{d} - f(\x_{\hat d})) \right) \right)=\\[6pt]
=- \frac{1}{|X|} \sum \limits_{x \in \y} \ln \left( N(\m_{\hat d},\Sigma_{\hat d})(\x_{\hat d}) \right) - \frac{1}{|X|} \sum \limits_{\x \in X} \ln\left( N(m_{d},\ssigma_{d})(x_{d} - f(\x_{\hat d})) \right)=
\\[6pt]
=H^{\times}(X_{\hat d}\|N(\m_{\hat d},\Sigma_{\hat d})) +H^{\times}(\ppp{d}{f} \| N(m_{d}, \ssigma_{d})).
\end{array}
$$
We can use Theorem~\ref{the:gaus_scross} for both summands separately: 
$$
\begin{array}{l}
H^{\times}(X\|\A_{f}(\R^{d-1},\R)) = H^{\times}(X_{\hat d}\|\G(\R^{d-1})) +H^{\times}(\ppp{d}{f} \| \G(\R)) =\\[6pt]
=\frac{d-1}{2}\ln(2 \pi e) + \frac{1}{2} \ln \left( \det \left(\cov(X_{\hat d}) \right) \right) + \frac{1}{2}\ln(2 \pi e) + \frac{1}{2} \ln \left( \frac{1}{n} \sum \limits_{\x \in X} ( x_{d} - f(\x_{\hat d}) - m_{d} )^2 \right).

\end{array}
$$

\end{proof}

As a corollary from the above theorem, we obtain that the optimal from the cross-entropy point of view function which describes a cluster can be obtained by a least squares method \cite{bjorck1996numerical}.

\begin{observation}
Let $X \subset \R^d$ be a data set and a family of functions $\F \subset \C( \R^{d-1}, \R ) $ be given. Then
$$
\argmin_{f \in \F} H^{\times}(X\|\A_f(\R^{d-1},\R^{d}))= \argmin_{f \in \F} \left\{ \sum \limits_{\x \in X} |x_{d}-f(\x_{\hat d})-m_{d}|^2 \right\},
$$
where $m_{d}=\mean(X_{d})$.
\end{observation}

Consequently, we minimize cross-entropy by finding a least squares estimation. Moreover, if $\F$ is a set of function which are invariant under the operations $f \to a+f$ for any $a$, it is enough to find 
$$
\argmin_{f \in \F} |x_{d}-f(\x_{\hat d})|^2.
$$

\begin{corollary}\label{col:en}
Let $X \subset \R^d$ be a data set, and let a family of functions $\F\subset \C( \R^{d-1}, \R ) $ be invariant under the operations $f \to a+f$ for $a \in \R$.
Let $\bar f \in \F$ be such that $ \bar f = \argmin \limits_{f \in \F} |x_{d}-f(\x_{\hat d})|^2 $. Then
\begin{eqnarray*}
\min_{f \in \F} H^{\times}(X\|\A_f(\R^{d-1},\R^{d}))= 
\frac{d}{2} \ln( 2 \pi e ) + \frac{1}{2}\ln( \det(\Sigma_{\hat d}) ) + \frac{1}{2} \ln \left(\frac{1}{n} \sum_{\x \in X} (x_{d}-\bar f (\x_{\hat d}))^2 \right),
\end{eqnarray*}
where
$\Sigma_{\hat d} = \cov(X_{\hat d})$.
\end{corollary}

The above theorem guarantees that the cost function is decreasing during iterations. The analogue of this result does not hold for AcaGMM (PCA is used for finding a local coordinate system). Consequently, in \afCEC (contrary to AcaGMM) we are able to construct a simple stop condition.

\subsection{Coordinate system in afCEC model}

In the previous subsection there was shown how to determine optimal parameters for one cluster in arbitrarily given coordinate system. Now we describe how to fit the optimal one for the \afCEC method.

In AcaGMM the PCA (Principal Component Analysis) was used for finding a locally adapted coordinate system.
Unfortunately, this operation causes problems with convergence (it is hard to construct a reasonable stop condition). More precisely, by using PCA we do not minimize a cost function which is connected with least squares estimation. By applying two methods (PCA and regression) separately we do not minimize any of them. 

In the case of \afCEC all computation use the canonical basis. We need only to decide which coordinate is chosen as dependent (then the rest becomes automatically explanatory).

\begin{figure}[htp]
\begin{center}
\subfigure[The c-type set and parabola fitted with assumption that $x$ is the dependent variable.]
{\label{fig:ex_cor1}
\includegraphics[width=2.5in]{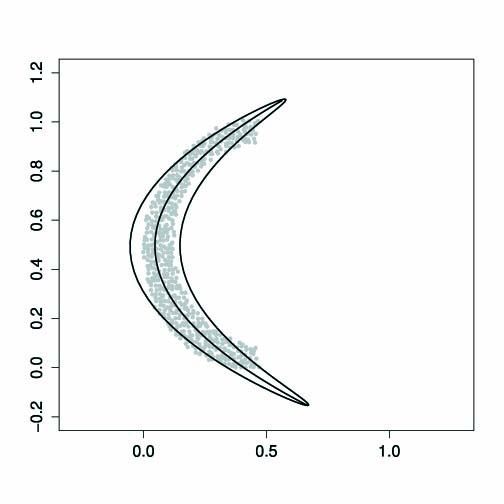}} \quad
\subfigure[The c-type set and parabola fitted with assumption that $y$ is the dependent variable.]
{\label{fig:ex_cor2}
\includegraphics[width=2.5in]{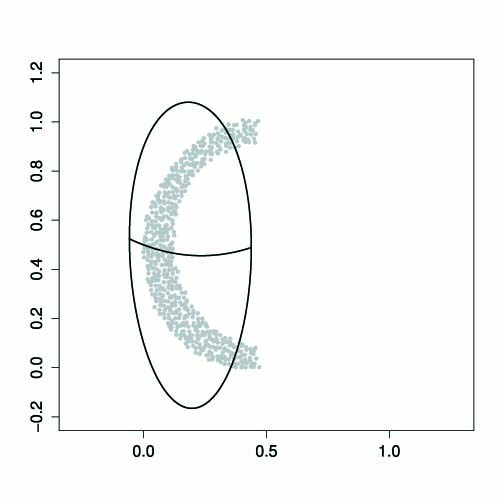}} 
\end{center}
\caption{ Estimation of f-adaptive Gaussian density in two different coordinates.}
\label{fig:ex_cor}
\end{figure}

Our intuition to verify all possible coordinates in the canonical basis came from the Implicit Function Theorem \cite{krantz2002implicit}.
More precisely, under reasonable assumptions, for an implicit function $F(\x) = 0$ where $F \colon \R^{d} \to \R$ and an arbitrary zero $\bar \x \in \R^d$ of $F$, we can 
find $k \in \{1,\ldots,d\}$ and $f:\R^{d-1} \to \R$ such that locally in the
neighborhood of $\bar \x$
$$
\begin{array}{l}
\{\x \in B(\bar \x,r):F(\x)=0\} \\[1ex]
=\{(x_1,\ldots,x_{k-1},f(\x_{\hat k}),x_{k+1},\ldots,x_d): \x_{\hat k} \in B(\x_{\hat k},r) \subset \R^{d-1}\}
\end{array}
$$
where $B(\bar \x,r)$ is a ball with center $\bar \x$ and radius $r$.

Consequently for data $X \subset \R^{d}$ we search for $k = 1, \ldots , d$ and $f$ such that $X$ can be optimally approximated by the set
$$
(x_1,\ldots, x_{k-1},f(\x_{\hat k}),x_{k+1},\ldots,x_{d}) \mbox{ for } \x \in X.
$$ 

\begin{example}
Let us consider a c-type set, see Fig. \ref{fig:ex_cor}. When using the canonical basis of $\R^2$, we have to consider two possible estimated curves (in our case parabolas).
We can treat $x$ as a dependent variable, see Fig. \ref{fig:ex_cor1}, or we choose $y$ as a dependent one, see Fig. \ref{fig:ex_cor2}. If we assume that dependent variable is $x$, we obtain the parabola $x= 1.4755y^2 -1.4602y + 0.4078$, and the sum of squared errors is equal to $1.420948$. On the other hand, if $y$ is the dependent coordinate, we have $y= 0.8x^2 -0.3756x + 0.4997$ with squared errors $53.35997$.
Consequently, the optimal coordinate system is describe by using $x$ as the dependent variable.
\end{example}

In the above example, we consider only $\R^2$ but in higher dimensional spaces we have to consider $d$ different possible choices of dependent variable: ($X_{\hat k},X_{ k})$ for $k \in 1, \ldots, d$. 

In conclusion, for one cluster $X \subset \R^d$ we can estimate parameters of the model in two steps. First, we consider all possible choices of dependent variable: functions $f_k$ (corresponding with relations $x_k=f(\x_{\hat k})$), means $m_{ k} = \mean(X_{k}^{f^k})$, $\m_{\hat k} = \mean(X_{\hat k})$ and covariances $\Sigma_{\hat k} =\cov(X_{\hat k})$, $ \Sigma_{k} = \cov(X_{k}^{f^k})$ for $k=1,\ldots,d$.
Then we determine the optimal dependent variable
$$
j = \argmin_{k=1,\ldots,d} \left\{ H^{\times}\left(X\|N([\m_{\hat k},m_{k}]^T,\Sigma_{\hat k},\ssigma_{k},f_{k}) \right) \right\}.
$$
Consequently, our data set is represented by the function, mean and covariance matrix 
$$
f=f_j
\qquad
\m=[\m_{\hat j},0], \qquad
\Sigma=
\begin{bmatrix}
\Sigma_{\hat j} & 0 \\
0 & \ssigma_{j}
\end{bmatrix}
$$
where subscript $j \in \{1,\ldots, d\}$ denotes the dependent variable in cluster.

The full algorithm can now be described. We use an adapted Lloyd's method which is based on the simultaneous application of two steps. 
First, we construct a new division of $X$ by matching each element  $\x \in X$ to a group such that the cost function is minimal.
Then, we estimate new parameters in each cluster by applying the method presented in previous subsection, see Algorithm \ref{alg1}.

\begin{algorithm}[htp] 
\caption{\afCEC:} 
\label{alg1} 
\begin{algorithmic} 
\STATE {\bf Input}
\STATE\hspace\algorithmicindent number of clusters $k > 0$
\STATE\hspace\algorithmicindent curve family $\F $
\STATE\hspace\algorithmicindent stop condition $\varepsilon > 0$
\STATE\hspace\algorithmicindent dataset $X$ ($d$ - dimension of data) 
\STATE {\bf initial conditions}
\STATE\hspace\algorithmicindent \textit{obtain} initial clustering $X_1,\ldots,X_k$ 
\STATE\hspace\algorithmicindent \textit{obtain} probabilities $ p_i = \frac{|X_i|}{| X |}$ for $i=1,\ldots,k$
\STATE\hspace\algorithmicindent \textit{obtain} parameters of each cluster $f_{i}$, mean $\m_{\hat i}$ and covariances $ \Sigma_{\hat i}$, $\Sigma_{i}$ in each cluster (choosing the best orientation)
\STATE\hspace\algorithmicindent \textit{obtain} cost function 
\STATE\hspace\algorithmicindent $ h_0 = \sum_{i=1}^{k} p_i( -\ln(p_i) + H^{\times}(X_i\| N([\m_{\hat i},0]^T, \Sigma_{\hat i }, \Sigma_{i }) ))  $
\REPEAT
\STATE\hspace\algorithmicindent $n=0$
\STATE\hspace\algorithmicindent \textit{obtain new} clustering $X_{1}, \ldots, X_{k}$ by matching elements to the cluster such that $ \left( - \ln(p_i) - \ln( N([\m_{\hat i},0]^T, \Sigma_{\hat i }, \Sigma_{i }) ) \right) $ is minimal
\STATE\hspace\algorithmicindent \textit{delete} unnecessary clusters ($|X_i| < 1 \% \cdot |X| $) by adding elements to the closest existing one 
\STATE\hspace\algorithmicindent \textit{update} parameter $k$
\STATE\hspace\algorithmicindent $n=n+1$
\STATE\hspace\algorithmicindent \textit{obtain new} probabilities $ p_i = \frac{|X_i|}{|X|}$ for $i=1,\ldots,k$
\STATE\hspace\algorithmicindent \textit{obtain new} parameters of each cluster $f_{i}$, mean $\m_{\hat i}$ and covariances $ \Sigma_{\hat i}$, $\Sigma_{i}$ in each cluster (choosing the best orientation) 
\STATE\hspace\algorithmicindent \textit{obtain new} cost function 
\STATE\hspace\algorithmicindent $ h_n = \sum_{i=1}^{k} p_i( -\ln(p_i) + H^{\times}(X_i\| N([\m_{\hat i},0]^T, \Sigma_{\hat i }, \Sigma_{i }) ))  $ 
\UNTIL{$h_n \geq h_{n-1} - \varepsilon$}
\end{algorithmic}
\end{algorithm}

\section{Experiments and analysis}

In this section we present a comparison of the \afCEC method with AcaGMM, GMM and CEC. 
Since AcaGMM is not a density model, the Log-likelihood function is not well-defined. Nevertheless, by the input the Jacobian of AcaGMM transformation, we obtain
a valid probability distribution, see Appendix.
\begin{figure}[htp]
\begin{center}
\subfigure[\afCEC]
{\label{fig:cir1}
\includegraphics[width=1.1in]{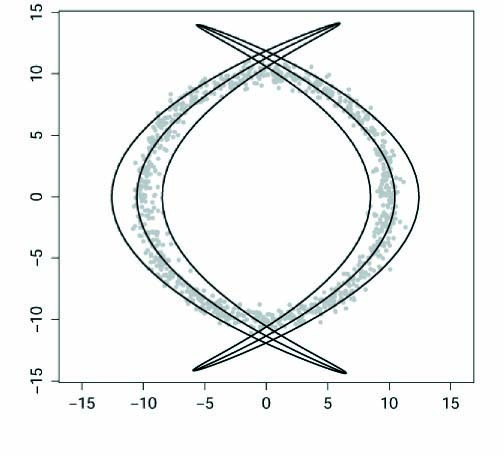}} 
\subfigure[AcaGMM]
{\label{fig:cir2}
\includegraphics[width=1.1in]{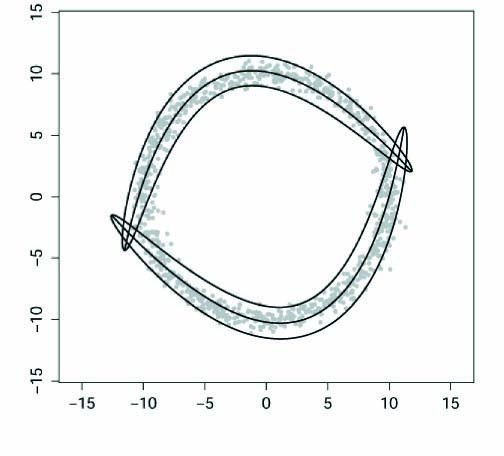}} 
\subfigure[GMM.]
{\label{fig:cir3}
\includegraphics[width=1.1in]{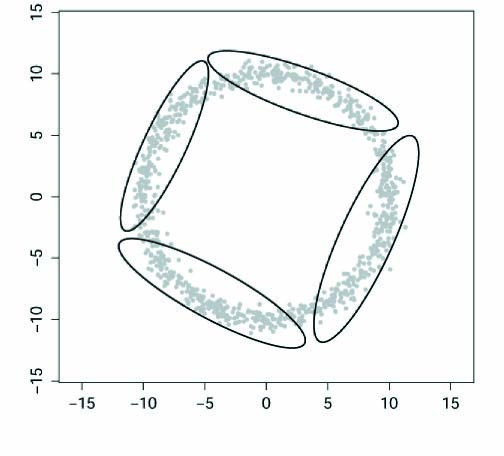}} 
\subfigure[CEC]
{\label{fig:cir4}
\includegraphics[width=1.1in]{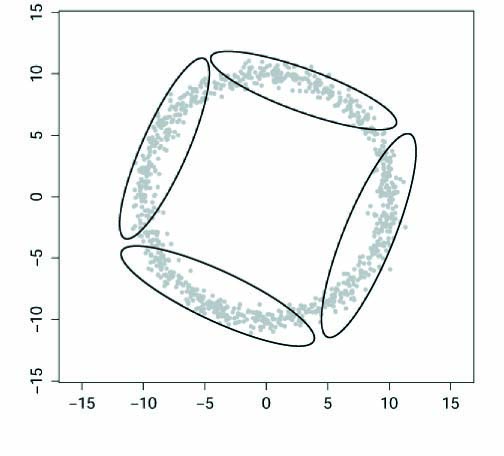}} \\
\subfigure[Evolution of Log-likelihood function when the number of clusters increases from $1$ to $10$.]
{\label{fig:cir5}
\includegraphics[width=1.5in]{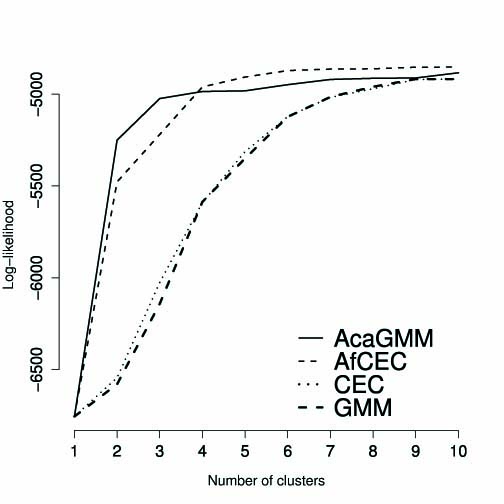}} \quad
\subfigure[Evolution of Log-likelihood function when the number of parameters increases from $1$ to $80$.]
{\label{fig:cir6}
\includegraphics[width=1.5in]{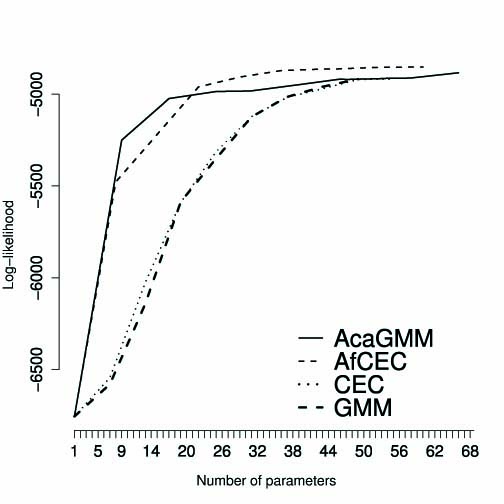}} \quad
\subfigure[Evolution of BIC function when the number of cluster increases from $1$ to $10$.]
{\label{fig:cir7}
\includegraphics[width=1.5in]{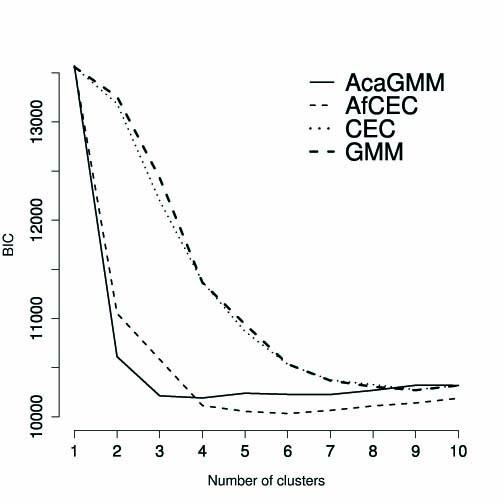}} 
\end{center}
\caption{ Results of afCEC, AcaGMM, CEC and GMM in the case of circle-type set.}
\label{fig:cir}
\end{figure}
To compare the results we use the standard Bayesian Information Criterion (BIC) 
$$
BIC = -2LL+k\log(n)
$$ 
and Akaike Information Criterion (AIC) 
$$
ACI = -2LL+2k,
$$ 
where $k$ is a number of parameters in the model, $n$ is a number of points, and $LL$ is a maximized value of the Log-likelihood function.
Consequently, we need a number of parameters which are used in each model. 
In a case of $\R^2$, AcaGMM uses two scalars for mean, three scalars for covariance matrix, two scalars for parabola and one for local coordinate system (obtained by PCA). On the other hand, in \afCEC we do not need scalar for the local coordinate system. Consequently, \afCEC uses two scalars for mean, three scalars for covariance matrix and two scalars for parabola\footnote{It should be emphasized that in \afCEC we need to remember which coordinate is the dependent one. This parameter is discrete so we do not consider it in our investigation.}. 

Let us start from a synthetic data set. First, we report the results of afCEC, AcaGMM, CEC and GMM in the case of a circle-type set, see Fig.~\ref{fig:cir}. 
Fig.~\ref{fig:cir5} shows how the Log-likelihood function changes when the number of clusters increases from $1$ to $10$. Similar relation, in respect to number of parameters\footnote{Plots which present relation between Log-likelihood functions and number of parameters was constructed by linear approximation of known values of the function.}, is presented in Fig.~\ref{fig:cir6}.  For a similar values of Log-likelihood function, we need 2 clusters in \afCEC and AcaGMM and 4 in GMM and CEC, see Fig.~\ref{fig:cir}. 
In such a case, the BIC criterion shows that algorithms which use curved densities model better fit data with using smaller number of parameters.
  
\begin{figure}[htp]
\begin{center}
\subfigure[AfCEC]
{\label{fig:sp1}
\includegraphics[width=1.1in]{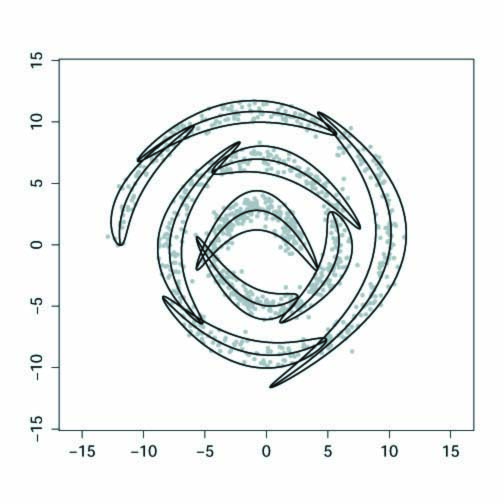}} 
\subfigure[AcaGMM]
{\label{fig:sp2}
\includegraphics[width=1.1in]{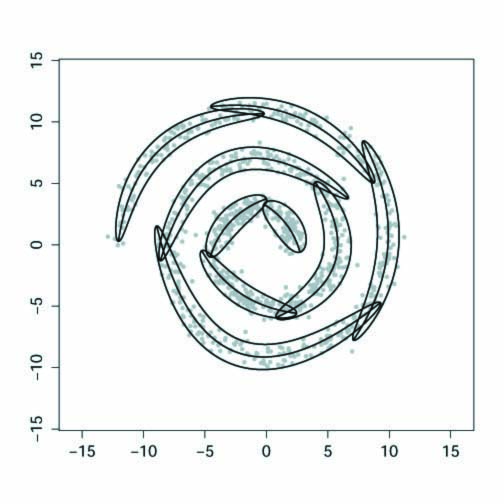}} 
\subfigure[GMM]
{\label{fig:sp3}
\includegraphics[width=1.1in]{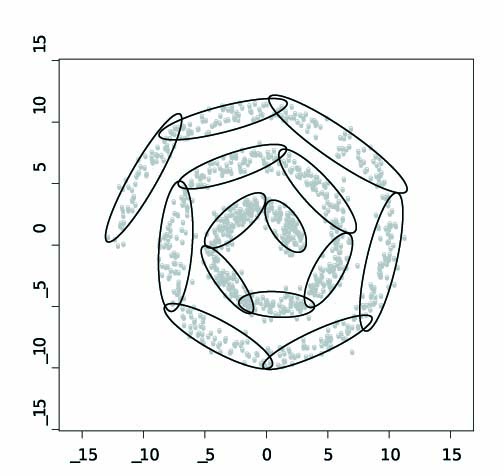}} 
\subfigure[CEC]
{\label{fig:sp3}
\includegraphics[width=1.1in]{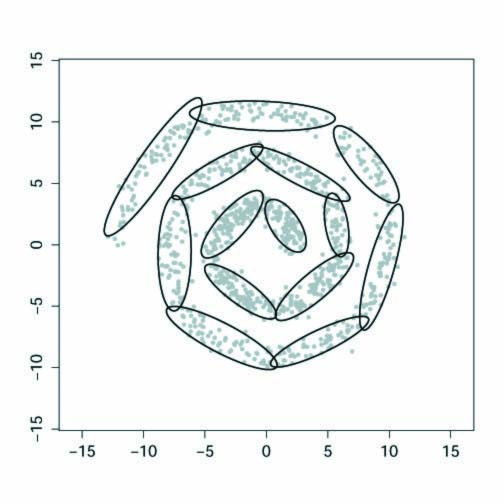}} 

\subfigure[Evolution of Log-likelihood function when the number of clusters increases from $1$ to $15$.]
{\label{fig:sp5}
\includegraphics[width=1.5in]{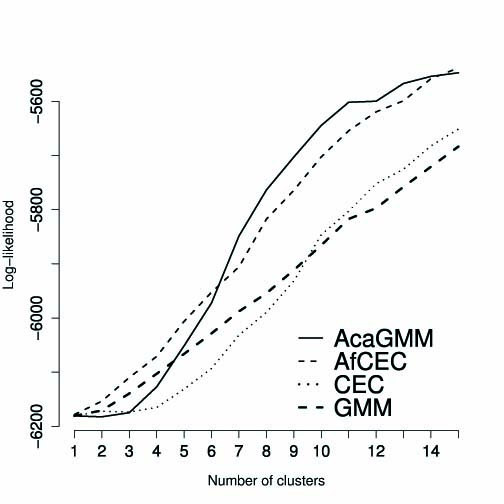}} \quad
\subfigure[Evolution of Log-likelihood function when the number of parameters increases from $1$ to $120$.]
{\label{fig:sp6}
\includegraphics[width=1.5in]{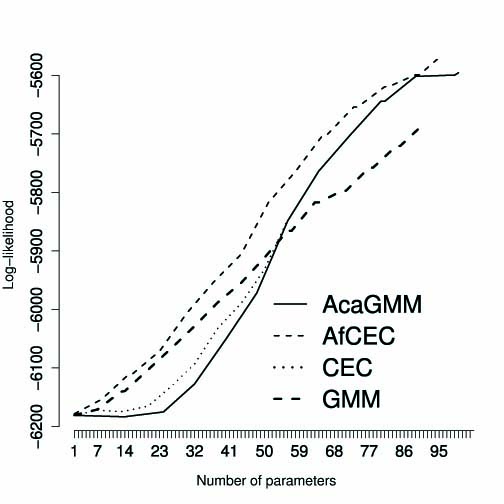}} \quad
\subfigure[Evolution of BIC function when the number of cluster increases from $1$ to $10$.]
{\label{fig:sp7}
\includegraphics[width=1.5in]{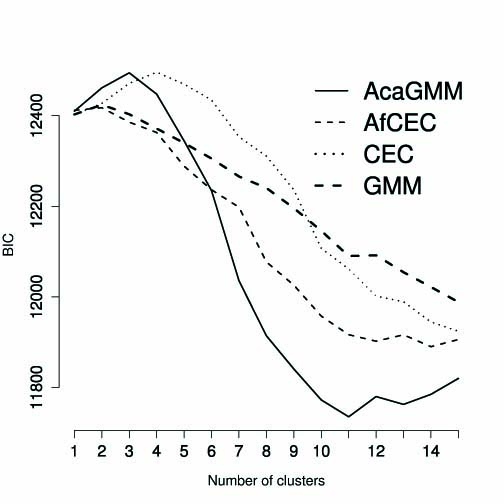}} 
\end{center}
\caption{ Results of afCEC, AcaGMM, CEC and GMM in the case of spiral-type set.}
\label{fig:sp}
\end{figure}

Similar situation can be observed in a more complex case of spiral-type set, see Fig. \ref{fig:sp}.
In Table \ref{tab:real_datasets_LL}, the mean and maximum value of Log likelihood for 100 initializations of algorithms are shown. As we see for a similar values of Log-likelihood function, we have to use 9 clusters for \afCEC and AcaGMM and 14 for GMM and CEC. The comparison of algorithms by using BIC and AIC with similar values of Log-likelihood function we present in Tab. \ref{tab:spirala_datasets}. 

Algorithms which are able to adapt to curve type structures (AcaGMM, \afCEC) better fit data. More precisely, the Log-likelihood function takes a larger value with the same number of parameters, see Fig.~\ref{fig:cir6} and Fig.~\ref{fig:sp6}. Since Log-likelihood increases with growing of the number of classes, we use BIC criterion which takes into account the number of parameters. In the case of AcaGMM and afCEC, we obtain optimal value of BIC after about 4-6 iterations.
In conclusion, AcaGMM and \afCEC better fit data (yield a higher value of Log-likelihood function) while require lower number of parameters. 

Algorithms AcaGMM and \afCEC give a comparable value of Log-likelihood, see Fig.~\ref{fig:cir5} and Fig.~\ref{fig:sp5}. Nevertheless, \afCEC uses less parameters, see Fig.~\ref{fig:cir6} and Fig.~\ref{fig:sp6}. Moreover, strong theoretical background of the method guarantees that the cost function decreases in each iteration. Consequently, we obtain a simple stop condition for our method.


\begin{table}[!t]
\centering
\scalebox{0.7}{ 
\begin{tabular}{ c | c c c | c c c | c c c | c c c }
& \multicolumn{3}{c}{\afCEC} & \multicolumn{3}{c}{AcaGMM} & \multicolumn{3}{c}{GMM} & \multicolumn{3}{c}{CEC} \\ 
& NP & mean LL & max LL & NP & mean LL & max LL & NP & mean LL & max LL & NP & mean LL & max LL \\
\hline
1 & 7 & -6178,30 & -6178,30 & 8 & -6180,86 & -6180,86 & 6 & -6180,68 & -6180,68 & 6 & -6180,68 & -6180,68 \\
2 & 14 & -6153,61 & -6069,15 & 16 & -6182,41 & -6104,58 & 12 & -6172,16 & -6157,13 & 12 & -6170,67 & -6127,52 \\
3 & 21 & -6109,99 & -6012,19 & 24 & -6174,88 & -6068,96 & 18 & -6173,61 & -6128,87 & 18 & -6139,47 & -6088,73 \\
4 & 28 & -6070,96 & -5924,87 & 32 & -6127,14 & -5987,66 & 24 & -6165,16 & -6062,66 & 24 & -6102,95 & -6041,99 \\
5 & 35 & -6006,17 & -5868,56 & 40 & -6051,35 & -5836,19 & 30 & -6131,26 & -6026,12 & 30 & -6066,26 & -5989,85 \\
6 & 42 & -5952,44 & -5713,61 & 48 & -5972,21 & -5667,10 & 36 & -6093,05 & -5990,69 & 36 & -6028,42 & -5953,28 \\
7 & 49 & -5905,57 & -5675,39 & 56 & -5848,57 & -5558,34 & 42 & -6031,69 & -5930,99 & 42 & -5987,86 & -5882,36 \\
8 & 56 & -5817,57 & -5612,98 & 64 & -5763,63 & -5511,39 & 48 & -5989,59 & -5868,69 & 48 & -5954,45 & -5865,29 \\
9 & 63 & -5764,29 & -5509,13 & 72 & -5702,82 & -5482,30 & 54 & -5931,64 & -5814,95 & 54 & -5911,52 & -5804,07 \\
10 & 70 & -5702,46 & -5494,73 & 80 & -5644,46 & -5460,11 & 60 & -5846,39 & -5741,61 & 60 & -5865,69 & -5766,89 \\
11 & 77 & -5654,33 & -5441,22 & 88 & -5601,65 & -5435,63 & 66 & -5802,84 & -5689,48 & 66 & -5817,09 & -5713,91 \\
12 & 84 & -5619,46 & -5410,99 & 96 & -5599,81 & -5448,04 & 72 & -5752,16 & -5636,58 & 72 & -5797,22 & -5664,69 \\
13 & 91 & -5598,93 & -5430,61 & 104 & -5566,99 & -5423,52 & 78 & -5725,24 & -5609,44 & 78 & -5757,77 & -5623,56 \\
14 & 98 & -5558,18 & -5384,77 & 112 & -5553,93 & -5420,68 & 84 & -5682,13 & -5542,43 & 84 & -5720,74 & -5563,87 \\
15 & 105 & -5538,44 & -5392,29 & 120 & -5547,13 & -5431,19 & 90 & -5651,20 & -5554,23 & 90 & -5683,41 & -5555,02 \\

\end{tabular}
}
\caption{ Comparison of the afCEC, CEC and GMM Chinese and Latin characters.}
\label{tab:real_datasets_LL}
\end{table}


\begin{table}
\centering
\begin{tabular}{ l | c c c c c }
\hline
Algorithms & Number of & Number of & Log-likelihood & BIC & AIC \\
& clusters & parameters & & & \\
\hline
\afCEC & 9 & 9$\cdot$7$=$63 & -5508.83 & 11452.85 &11143.66 \\
AcaGMM & 9 & 9$\cdot$8$=$72 & -5497.11 & 11491.58 & 11138.22 \\
GMM & 14 & 14$\cdot$6$=$84 & -5520.96 & 11622.17 & 11209.92\\
CEC & 14 & 14$\cdot$6$=$84 & -5510.09 & 11600.44 & 11188.18 \\
\end{tabular}
\caption{Comparison of afCEC, AcaGMM, CEC and GMM in the case of spiral-type set, see Fig. \ref{fig:sp}.}
\label{tab:spirala_datasets}
\end{table}

Chinese characters mainly consist of straight-line strokes
(horizontal, vertical) and curve strokes (slash, backslash and
many types of hooks). GMM has already been employed for
structure analysis of Chinese characters, and achieves commendable
performance \cite{zhang2004competitive}. However, some lines extracted
by GMM may be too short and is quite difficult to
join these short lines to form semantic strokes due to the
ambiguity of joining. This problem becomes more serious
when analyzing handwritten characters by GMM, and this was the
motivation to use AcaGMM to represent Chinese characters.
In Tab.~\ref{tab:real_datasets}, we present a comparison of afCEC, AcaGMM, GMM and CEC on Chinese and Latin characters: \begin{CJK}{UTF8}{gbsn}
犬 (dog), 乞 (beg), 父 (father), 仉 (mother), 火 (fire), 主 (master), b, R, S. The number of clusters has been determined so as to obtain a similar value of Log-likelihood function.
\end{CJK}


\begin{table}[!t]
\centering
\scalebox{0.7}{ 
\begin{tabular}{ c | c c c | c c c | c c c | c c c }
& \multicolumn{3}{c}{\afCEC} & \multicolumn{3}{c}{AcaGMM} & \multicolumn{3}{c}{GMM} & \multicolumn{3}{c}{CEC} \\ 
& NP & LL & BIC & NP & LL & BIC & NP & LL & BIC & NP & LL & BIC \\
\hline
\begin{CJK}{UTF8}{gbsn}
犬 
\end{CJK}
& 5 & 1148.57 & -2071.93 & 5 & 1030.02 & -1802.66 & 7 & 1060.26 & -1850.18 & 7 & 1015.29 & -1760.33 \\
\begin{CJK}{UTF8}{gbsn}
乞 
\end{CJK}
& 5 & 1000.78 & -1770.42 & 5 & 959.71 & -1655.26 & 7 & 1170.85 & -2064.33 & 7 & 1175.96 & -2074.55 \\
\begin{CJK}{UTF8}{gbsn}
父 
\end{CJK}
& 4 & 1009.02 & -1836.69 & 4 & 880.32 & -1553.38 & 5 & 824.51 & -1454.71 & 5 & 811.76 & -1429.21 \\
\begin{CJK}{UTF8}{gbsn}
父 
\end{CJK}
& 4 & 1009.02 & -1836.69 & 4 & 880.32 & -1553.38 & 6 & 1027.55 & -1821.93 & 6 & 1032.92 & -1832.67 \\
\begin{CJK}{UTF8}{gbsn}
仉 
\end{CJK}
& 6 & 1329.01 & -2372.74 & 6 & 1272.74 & -2219.45 & 8 & 1364.94 & -2403.85 & 8 & 1422.27 & -2518.51 \\
\begin{CJK}{UTF8}{gbsn}
火 
\end{CJK}
& 4 & 1045.53 & -1911.53 & 4 & 921.65 & -1638.12 & 5 & 900.25 & -1608.15 & 5 & 902.12 & -1611.89 \\
\begin{CJK}{UTF8}{gbsn}
火 
\end{CJK}
& 4 & 1045.53 & -1911.53 & 4 & 921.65 & -1638.12 & 6 & 1017.13 & -1803.44 & 6 & 1018.31 & -1805.79 \\
\begin{CJK}{UTF8}{gbsn}
主 
\end{CJK}
& 5 & 1011.27 & -1794.47 & 5 & 962.93 & -1665.21 & 7 & 1079.99 & -1840.69 & 7 & 1181.03 & -2042.77 \\
b
& 3 & 2660.87 & -5158.99 & 3 & 2738.24 & -5290.49 & 4 & 2686.59 & -5187.19 & 4 & 2678.49 & -5170.99 \\
R
& 3 & 1911.73 & -3652.67 & 3 & 1578.61 & -2962.04 & 4 & 1996.56 & -3797.94 & 4 & 1989.31 & -3783.43 \\
S
& 3 & 1883.88 & -3604.71 & 3 & 1907.83 & -3629.32 & 4 &1875.93 & -3565.52 & 4 & 1866.01 & -3545.68 \\
\end{tabular}
}
\caption{Comparison of the afCEC, AcaGMM, CEC and GMM methods for Chinese and Latin characters.}
\label{tab:real_datasets}
\end{table}



\section{Appendix--AcaGMM Gaussian model}

As it was previously mentioned, AcaGMM does not use densities. More precisely, the Jacobian of the transformation was not taken into consideration. However, the EM procedure, which was used in AcaGMM, works with probability distributions. Therefore, from the theoretical point of view the above procedure is incorrect. Moreover, if we want to compare our method by using of the Log-likelihood function we need densities. 

Let us start from numerical integration of the original AcaGMM function and of the model rescaled by Jacobian correction. The Simpson method \cite{james1985applied}, on the square $[-5,5] \times [-5,5]$ with $50 000$ segments was used. 
The integral in the case of AcaGMM is equal to $1.038$. After correction we obtain $1$ (with a precision of $10^4$).

Let us consider situation of the AcaGMM model.
Suppose $X$ and $Y$ are zero mean independent Gaussian distributions with variances $\sigma_1, \sigma_2$:
$$
N_{XY}(x,y)=
\frac{1}{\sqrt{2 \pi} \sigma_1 \sigma_2}
\exp\left(-\frac{x^2+y^2}{2 \sigma_1 \sigma_2} \right).
$$
Moreover, let
$$
Z = g(X,Y), \quad W = h(X,Y),
$$
where $g,h \in \C(\R^2,\R)$.
Let $J(x,y)$ represent the Jacobian of the original 
transformation 
$$
J(x,y)=\det
\begin{bmatrix}
\frac{\partial g(x,y)}{\partial x} & \frac{\partial g(x,y)}{\partial y} \\
\frac{\partial h(x,y)}{\partial x} & \frac{\partial h(x,y)}{\partial y}
\end{bmatrix}.
$$
In such a case, we have
$$
N_{ZW}(z,w) = 
\sum_{\{(x,y)\in \R^2 \colon (g(x,y),h(x,y)) = (z,w)\}} \frac{ N_{XY}(x,y)}{| J (x,y) |}.
$$

Let us consider the function $f$ expressed as parametric equation  
 $f:=\{(x(t),y(t)) \colon t \in \R\}$ (in the case of AcaGMM it is a parabola). Using the formula from \cite[Table 1.]{zhang2005active} we obtain the orthogonal projection $(x(t_0),y(t_0))$ of point $(p_1,p_2)$ on curve $f$: 
$$
t_0= p_f(p_1,p_2)=
\left\{ 
\begin{array}{l l }
\sqrt[3]{R + \sqrt{D}} + \sqrt[3]{ R - \sqrt{D}} & D>0  \\
0, \ 0, \ 0 & D = 0, Q=R=0 \\
2\sqrt{-Q}, \ -\sqrt{-Q}, \ -\sqrt{-Q} & D = 0, Q \neq 0, R \neq 0 \\
2\sqrt{-Q} \ cos(\frac{\phi+2i\pi}{3}), i=0,1,2, & D<0 \\
\mbox{ where } \phi=acos \left(\frac{R}{\sqrt{-Q^3}}\right) & 
\end{array} \right.
$$ 
where
$
Q= \frac{ 1 - 2 a p_2 }{6 a^2},
$
$
R = \frac{p_1}{4a^2}
$ and
$
D= Q^3+R^2.
$ 

On the other hand, the arc length of $f$ between zero and $(x(t_0),y(t_0))$ \cite[Formula (10)]{zhang2005active} is given by
$$
l(t_0) = \frac{1}{2}|t_0|\sqrt{ 1 + 4a^2 t_0^2 }+
\frac{1}{4a} \ln\left( 2|a| t_0 + \sqrt{ 1 + 4 a^2 t_0^2 } \right).
$$
Consequently, we have
$$
g^{-1}(p_1,p_2)=\| (x(t_0),y(t_0)) - (p_1,p_2)\|, 
$$
$$
h^{-1}(p_1,p_2)= l(t_0),
$$
where $ t_0= p_f(p_1,p_2)$.

\begin{figure}[htp]
\begin{center} 
\begin{tikzpicture}[scale=4]
\draw [<->, help lines] (1.5,1) -- (1.5,0) -- (2.5,0);
\draw [ultra thick,black,domain=1.5:2.5] plot (\x, {(1*(\x-1.5)^2 });

\draw [fill] (1.5+0.4,0.6) circle [radius=0.02];
\node[opacity=1] at (1.5+0.4,0.6+.1) {$(p_1,p_2)$};
\draw [fill] (1.5+0.6416396472194656,0.6416396472194656^2) circle [radius=0.02];
\node[opacity=1] at (1.8+0.6416396472194656,0.6416396472194656^2) {$(x(t),y(t))$};
\draw [ black] (1.5+0.6416396472194656,0.6416396472194656^2) -- (1.5+0.4,0.6);
\node[opacity=1] at (1.98,0.45) {$p$};
\node[opacity=1] at (1.8,0.2) {$l$};

\draw [<->, help lines] (0,1) -- (0,0) -- (1,0);
\draw [ultra thick,black,domain=0:1] plot (\x, {(0 });

\draw [fill] (0.7889957113463548,0.3063430560334737) circle [radius=0.02];
\node[opacity=1] at ((0.7889957113463548,0.3063430560334737+.1) {$\left(h^{-1}(p_1,p_2),g^{-1}(p_1,p_2) \right)$};
\draw [fill] (0.7889957113463548,0) circle [radius=0.02];
\draw [ black] (0.7889957113463548,0.3063430560334737) -- (0.7889957113463548,0);
\node[opacity=1] at (.74,0.15) {$p$};
\node[opacity=1] at (.4,0.1) {$l$};
\draw [ultra thick,black,->] (0.7,1.1)--(2.1,1.1);
\node[opacity=1] at ((1.4,1.2) {$(h(x,y),g(x,y))$};

\end{tikzpicture}
\end{center}
\caption{The transformation used in AcaGMM.}
\label{fig:AcaGMMFix}
\end{figure}
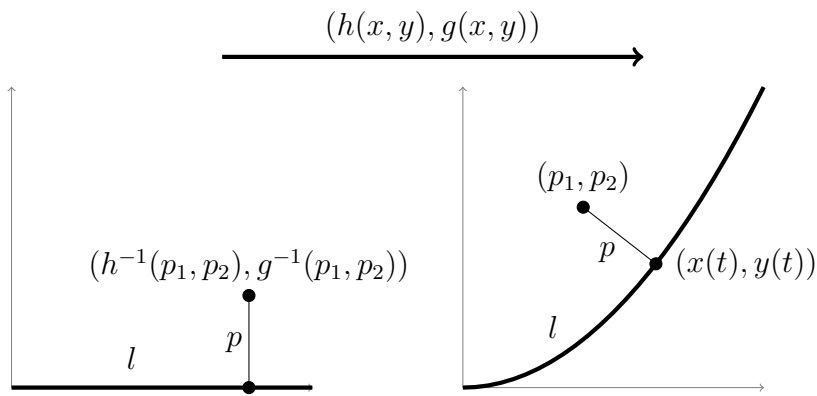

Our goal is to determine the Jacobian of our transformation, see Fig. \ref{fig:AcaGMMFix}. Let us consider an arbitrary small neighborhood of $(x(t_0),y(t_0))$. In such a case, the local curvature of $f$ at $(x(t_0),y(t_0))$ is the same as the curvature of the osculating circle\footnote{In differential geometry of curves, the osculating circle of a sufficiently smooth plane curve at a given point p on the curve has been traditionally defined as the circle passing through p and a pair of additional points on the curve infinitesimally close to p. Its center lies on the inner normal line, and its curvature is the same as that of the given curve at that point. This circle, which is the one among all tangent circles at the given point that approaches the curve most tightly, was named circulus osculans (Latin for ``kissing circle") by Leibniz.} at $(x(t_0),y(t_0))$. 

The radius of curvature in the case of parametric form of curve is given by
$$
r=\frac{(x'^2+y'^2)^{\frac{3}{2}}}{x'y''-y'x''}.
$$
Consequently, our goal is to determinate how a set is changing under the influence of the transformation, see Fig. \ref{fig:AcaGMMFix1}.

A small square neighborhood of the point $(p_1,p_2)$ is mapped to a trapezoid (asymptotically when a size of square converges to zero). 
This operation is showed in Fig. \ref{fig:AcaGMMFix1}.
It is easy to see that the square area changes linearly depending on the distance $p$. 
If we consider the situation where $p=r$, we obtain that our square is collapsed to a point. Consequently, for points above the curve Jacobian is
asymptotically proportional to
$$
\frac{r-p}{r}=1 - \frac{p}{r}.
$$
In a natural way, if a point $(p_1,p_2)$ is under the curve, the square area is increasing under the influence of the transformation. Therefore, the Jacobian is asymptotically equal to
$$
\frac{r+p}{r}=1 + \frac{p}{r}.
$$

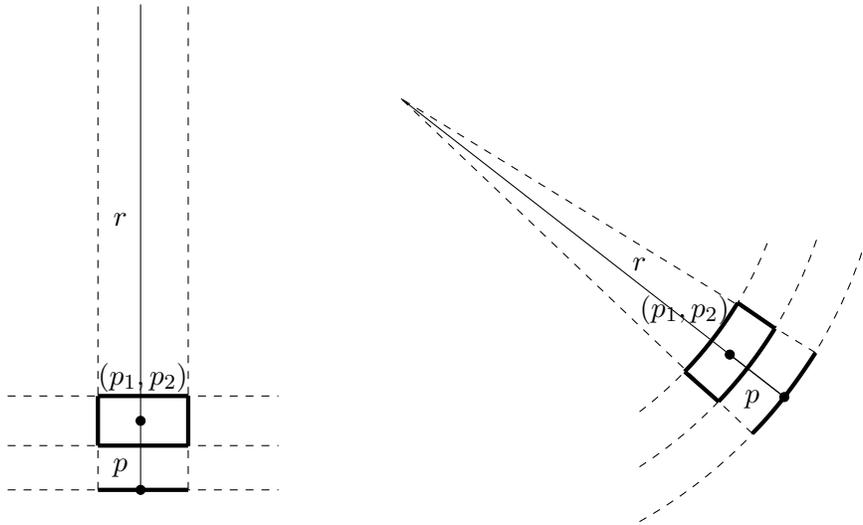
\begin{figure}[htp]
\begin{center} 
\begin{tikzpicture}[scale=3]
\draw [ultra thick,black,domain=2:2.28] plot (\x, {(1*(\x-1.5)^2 });
\draw [dashed] (2,0.5^2) -- (1.5-1.05666,1.7351);
\draw [dashed] (2.28,0.78^2) -- (1.5-1.05666,1.7351);

\draw [fill] (1.5+0.4,0.6) circle [radius=0.02];
\node[opacity=1] at (1.7,0.8) {\small $(p_1,p_2)$};
\node[opacity=1] at (1.5,1) {\small $r$};
\node[opacity=1] at (2,0.4) {\small $p$};
\draw [fill] (1.5+0.6416396472194656,0.6416396472194656^2) circle [radius=0.02];
\draw [ black] (1.5+0.6416396472194656,0.6416396472194656^2) -- (1.5+0.4,0.6);
\draw [ black] (1.5+0.6416396472194656,0.6416396472194656^2) -- (1.5-1.05666,1.7351);
\draw [dashed,black,domain=1.5:2.52] plot (\x, {(-sqrt(2.15305^2-(\x-(1.5-1.05666))*(\x-(1.5-1.05666)) )+1.7351});
\draw [ultra thick,black,domain=1.85:2.1] plot (\x, {(-sqrt(1.945^2-(\x-(1.5-1.05666))*(\x-(1.5-1.05666)) )+1.7351});
\draw [dashed,black,domain=1.5:2.29] plot (\x, {(-sqrt(1.945^2-(\x-(1.5-1.05666))*(\x-(1.5-1.05666)) )+1.7351});

\draw [ultra thick,black,domain=1.7:1.935] plot (\x, {(-sqrt(1.745^2-(\x-(1.5-1.05666))*(\x-(1.5-1.05666)) )+1.7351});
\draw [dashed,black,domain=1.5:2.07] plot (\x, {(-sqrt(1.745^2-(\x-(1.5-1.05666))*(\x-(1.5-1.05666)) )+1.7351});

\draw [ultra thick,black] (1.7,0.524385) -- (1.85,0.391846);
\draw [ultra thick,black] (1.935,0.829575) -- (2.1,0.716031);
\draw [ultra thick,black,domain=-0.4-0.5:0-0.5] plot (\x, {(0 });
\draw [dashed] (-0.8-0.5,0) -- (0.4-0.5,0);
\draw [dashed] (-0.4-0.5,0) -- (-0.4-0.5,2.15305);
\draw [dashed] (0-0.5,0) -- (0-0.5,2.15305);

\draw [ultra thick,black] (-0.4-0.5,0.3063430560334737+0.11) -- (0-0.5,0.3063430560334737+0.11);
\draw [ultra thick,black] (-0.4-0.5,0.3063430560334737-0.11) -- (0-0.5,0.3063430560334737-0.11);
\draw [ultra thick,black] (-0.4-0.5,0.3063430560334737+0.11) -- (-0.4-0.5,0.3063430560334737-0.11);
\draw [ultra thick,black] (0-0.5,0.3063430560334737-0.11) -- (0-0.5,0.3063430560334737+0.11);

\draw [dashed] (-0.8-0.5,0.3063430560334737+0.11) -- (0.4-0.5,0.3063430560334737+0.11);
\draw [dashed] (-0.8-0.5,0.3063430560334737-0.11) -- (0.4-0.5,0.3063430560334737-0.11);
\node[opacity=1] at (-0.8,0.1) {\small $p$};
\node[opacity=1] at (-0.8,1.2) {\small $r$};
\node[opacity=1] at (-0.7,0.5) {\small $(p_1,p_2)$};
\draw [fill] (0.7889957113463548-1.5,0.3063430560334737) circle [radius=0.02];
\draw [fill] (0.7889957113463548-1.5,0) circle [radius=0.02];
\draw [black] (0.7889957113463548-1.5,0) -- (0.7889957113463548-1.5,0.3063430560334737);
\draw [black] (0.7889957113463548-1.5,2.15305) -- (0.7889957113463548-1.5,0.3063430560334737);

\end{tikzpicture}

\end{center}
\caption{Transformation of a square neighborhood of a point $(p_1,p_2)$ under the influence of the AcaGMM function.}
\label{fig:AcaGMMFix1}
\end{figure}

\begin{figure}[htp]
\begin{center} 
\begin{tikzpicture}[scale=0.5]
\draw [<->, help lines] (0,9) -- (0,0) -- (5,0);
\draw [-, help lines] (0,-1) -- (0,0) -- (-5,0);
\draw [ultra thick,black,domain=-3:3] plot (\x, {(\x)^2 });

\draw [fill] (4,4) circle [radius=0.1];
\node[opacity=1] at (4,4.5) {$(4,4)$};

\draw [fill] (-1.459261,2.129444) circle [radius=0.1];
\draw [ black] (-1.459261,2.129444) -- (4,4);

\draw [fill] (-0.6498321,0.4222817) circle [radius=0.1];
\draw [ black] (-0.6498321,0.4222817)-- (4,4);

\draw [fill] (2.109093,4.448275) circle [radius=0.1];
\draw [ dashed] (2.109093,4.448275)-- (4,4);

\draw [fill] (0,6) circle [radius=0.1];
\node[opacity=1] at (0,6.5) {$(0,6)$};

\draw [fill] (2.345208,5.5) circle [radius=0.1];
\draw [ black] (2.345208,5.5) -- (0,6);

\draw [fill] (-2.345208,5.5) circle [radius=0.1];
\draw [ black] (-2.345208,5.5) -- (0,6);

\draw [fill] (0,0) circle [radius=0.1];
\draw [ black] (0,0) -- (0,6);

\draw [fill] (-4,1) circle [radius=0.1];
\node[opacity=1] at (-5,1.5) {$(-4,1)$};

\draw [fill] (-1.391769,1.93702) circle [radius=0.1];
\draw [ dashed] (-1.391769,1.93702) -- (-4,1);

\end{tikzpicture}
\end{center}
\caption{Position of the point and its orthogonal projection on the parabola $f(x) = x^2$. The distance between point and his orthogonal projection, when it is situated above the curve is marked by a solid line. On the other hand, if the relationship is reversed, we mark the projection by a dashed line.}
\label{fig:AcaGMMFix2}
\end{figure}
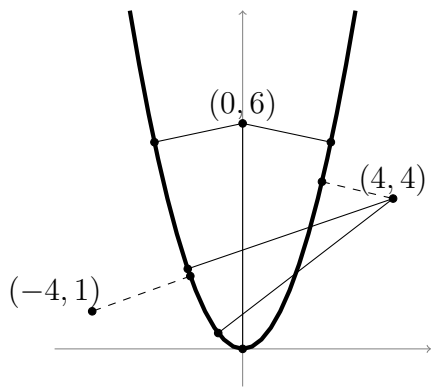

Now we have the formula for the Jacobian of AcaGMM transformation, but it depends on the relation between a point and its orthogonal projection. More precisely, we have to verify which formula should be used (or equivalently on which side of parabola a point is found), see Fig. \ref{fig:AcaGMMFix2}. 

We can easily verify where the point $(p_1,p_2)$ is in relation to the orthogonal projection $(x(t_0),y(t_0))$ by checking the orientation of a basis containing the  normal vector $(p_1,p_2)-(x(t_0),y(t_0))$ and the tangent vector $(x'(t_0),y'(t_0))$ at a point $(x(t_0),y(t_0))$. 
Consequently, we have to verify the sign of the determinant
$$
\det \left(
\begin{bmatrix}
p_1 - x(t_0) & x'(t_0) \\
p_2 - y(t_0) & x'(t_0)
\end{bmatrix}
\right).
$$

\section{Acknowledgements}

The study is cofounded by the European Union from resources of the European Social
Fund. Project PO KL ``Information technologies: Research and their interdisciplinary
applications'', Agreement UDA-POKL.04.01.01-00-051/10-00.



%
%

%
\end{document}